\documentclass{fcs}

\ClearShipoutPictureBG
\ExplSyntaxOn
\tl_set:Nn \l__fcs_doi_url_tl {}
\tl_set:Nn \l__fcs_long_detailed_page_tl {}
\tl_set:Nn \l__fcs_short_detailed_page_tl {}
\tl_set:Nn \l__fcs_copyright_tl {}
\RenewDocumentCommand \maketitle { }
  {
    \thispagestyle { plain }
    \twocolumn
      [
        \skip_vertical:N -6mm
        \group_begin:
          \fontsize{18pt}{14pt} \selectfont \bfseries \g__fcs_title_tl \par
        \group_end:
        \skip_vertical:N 5mm
        \seq_use:Nn \g__fcs_authors_seq {,~}
        \skip_vertical:N 5mm
        \if_int_compare:w 1 < \seq_count:N \g__fcs_authors_seq
          \tl_gset:Nn \g__fcs_short_author_tl { \g__fcs_first_author_tl \ \l__fcs_etal_tl}
        \else:
          \tl_gset:Nn \g__fcs_short_author_tl \g__fcs_first_author_tl
        \fi:
        \color_group_begin:
        \color_select:n { black!70 }
        \seq_use:Nn \g__fcs_addresses_seq { \par }
        \color_group_end:
        \skip_vertical:N 5mm
        \color_group_begin:
        \color_select:n { black!70 }
        E-mail:\ \g__fcs_corr_email_tl .
        \IfBooleanT { \g__fcs_has_equal_contrib_bool }
          { \quad *\ These\ authors\ contributed\ equally\ to\ this\ work. }
        \color_group_end:
        \par
        \skip_vertical:N 5mm
        \g__fcs_abstract_tl
        \skip_vertical:N 5mm
        \l__fcs_keywords_tl
        \skip_vertical:N 5mm
      ]
  }
\ExplSyntaxOff

\title{Output Discipline Sets the Floor: A Falsification-Guided, Model-Agnostic Agent Protocol for Auditable Hypothesis-Driven Decision Making}

\author[1,+]{Jing Liu}
\author[1]{Yongxing Qi}
\author[2]{Ping Yao}
\author[1]{Ting Wu}
\address[1]{Hangzhou Innovation Institute, Beihang University, Hangzhou, China}
\address[2]{32007 Unit, Beijing, China}

\corremail{hiiliujing@buaa.edu.cn}

\begin{abstract}
On auditable, hypothesis-driven decision tasks---where a candidate hypothesis space can be enumerated and every candidate checked against the evidence---what actually binds the performance of large language models? On the public TCMEval-SDT benchmark for Traditional Chinese Medicine syndrome differentiation (200 cases, up to three repetitions per configuration, paired tests), we find that the binding constraint is \emph{output discipline and reasoning structure, not model capability alone}: two years of model scaling and generation turnover move zero-shot scores by only $-0.03$--$+0.08$, while output-discipline failures---non-terminating reasoning chains, format non-compliance, run-to-run spread spanning up to 0.39 of the score range---persist even in frontier models. We test this hypothesis with a model-agnostic agent framework that formalizes decision making as a five-stage closed loop---hypothesis generation, verification, causal analysis, action, and feedback---enforced by a candidate $\times$ evidence matrix falsification protocol: every candidate is judged against every evidence item (support / oppose / irrelevant), at least half of the candidates must be falsified before a conclusion is licensed, and the surviving most-supported candidates are emitted in a machine-verifiable single-line format. Same-model controlled comparisons across five base models show the protocol's effect is sign-conditional on zero-shot output discipline: it rescues models whose discipline fails (DeepSeek-V4-Flash: $+0.267$/$+0.271$, decomposing entirely into termination rescue and robust on budget-certified rescue subsets ($\Delta=+0.24$ to $+0.46$ across certification variants; Appendix~\ref{sec:budget-neutralization})), reduces the run-to-run standard deviation of a high-variance model by an order of magnitude (Kimi K3: 0.218 $\rightarrow$ 0.007 on syndrome), and significantly costs already-disciplined models ($-0.087$ to $-0.120$ on budget-matched, temperature-controlled Qwen3-4B/32B). Format-only ablations attribute the cost to the candidate-bounded format layer and reveal a sign-flipping complementarity between the causal chain and the format layer; knowledge ablations show the protocol's measured effect is independent of the injected knowledge content---a protocol--knowledge decoupling that supports the framework's knowledge-as-revisable-data design. A two-stage adaptation transfers the protocol to misinformation-robust mathematical reasoning, detecting injected false premises at 89.8\% recall with a 6.8\% false-positive rate---against 100\% for a flag-everything baseline (Matthews correlation coefficient 0.54 vs.\ 0)---and restoring accuracy to the no-misinformation ceiling on a comparatively robust model with a small degradation headroom (0.039). Where zero-shot output discipline is weak, the reasoning protocol sets the floor; base-model capability always sets the ceiling.
\end{abstract}

\begin{document}

\section{Introduction}
\label{sec:intro}

Many high-stakes decisions follow a structured epistemic pattern: a decision maker collects evidence, proposes candidate hypotheses, verifies or falsifies each against the evidence, analyzes the causal chain connecting evidence to conclusion, and only then acts. Medical diagnosis is the canonical instance---whether in the Western tradition of differential diagnosis or, as in this paper's evaluation domain, the Traditional Chinese Medicine (TCM) process of \emph{bianzheng lunzhi} (syndrome differentiation and treatment): a physician collects evidence through the four examinations (inspection, auscultation/olfaction, inquiry, and palpation including pulse-taking), proposes candidate syndromes (\emph{zheng}), verifies or falsifies each hypothesis against the evidence, analyzes the causal chain from etiology to pathogenesis to syndrome, and only then determines the treatment principle and prescription. Epistemologically this is a closed loop of abduction (hypothesis generation), verification (hypothesis testing against evidence), causal analysis (deductive reasoning over domain logic, in the sense of causal modeling~\cite{pearl}), and action with feedback---and it recurs wherever decisions must be justified: clinical reasoning, fact checking, fault diagnosis, and scientific hypothesis testing.

Large language models (LLMs) have shown promise on such tasks, yet applying them exposes four persistent deficiencies. First, association-style direct mapping: given task inputs, models tend to jump directly to a conclusion by reproducing co-occurrence patterns from training corpora, rather than performing stepwise causal derivation---the shortcut-learning pathology documented from deep networks generally~\cite{geirhos2020shortcut} to LLMs specifically~\cite{du2024shortcut}. Second, absence of falsification: models neither propose a candidate hypothesis space nor verify each candidate against the evidence; conclusions are produced in a single pass and cannot be checked, reproduced, or audited, and left to critique itself a model does not fix this---intrinsic self-correction without external feedback fails to improve and can even degrade reasoning~\cite{huang2024selfcorrect,kamoi2024selfcorrection}. Without candidate-level verification the reasoning chain itself is also exposed to manipulation---injected false premises propagate along the chain and degrade conclusions even when the model internally possesses the correct knowledge~\cite{misinfoprop}. The sharpest form of this observation is that \emph{the knowledge is present and the instruction is present, yet correction still fails}: explicit correction instructions succeed in fewer than half of cases~\cite{misinfoprop}---what is missing is neither knowledge nor willingness but an enforced verification structure. Third, safety and reliability risks: hallucinated entities, quantities, and even dangerous recommendations pass without any hard verification---hallucination is a documented, detectable failure mode of LLM generation~\cite{farquhar2024semantic}, not a hypothetical concern. Fourth, deployment tension: cloud models raise data-privacy concerns, while on-premises models preserve privacy but inherit every deficiency above.

Existing work addresses these issues partially. Retrieval-augmented generation (RAG) grounds answers in knowledge bases but leaves the reasoning path free-form. Multi-agent medical systems distribute roles across agents but constrain the dialogue process rather than the reasoning path itself, at multiplied inference cost. Chain-of-thought prompting elicits reasoning text but offers no machine-verifiable checkpoint of correctness.

A striking empirical pattern motivates this paper. On the benchmark we adopt, two years of model scaling and generation turnover move zero-shot scores by only $-0.03$--$+0.08$ (Section~\ref{sec:progress}), while output-discipline failures---non-terminating reasoning chains, format non-compliance, and run-to-run spread spanning up to 0.39 of the score range (operationalized in Table~\ref{tab:discipline})---remain rampant even in frontier models (Section~\ref{sec:results}). We therefore hypothesize that on such tasks \emph{the binding constraint is output discipline and reasoning structure, not model capability alone}. The hypothesis is falsifiable and sign-conditional: a constrained-reasoning protocol should help exactly where output discipline is missing, and it should add pure, measurable overhead where discipline is already internalized. Testing it requires an instrument that enforces reasoning structure on an arbitrary base model without parameter updates; we build one.

The instrument is a model-agnostic agent framework that formalizes the epistemic loop above as a data-processing protocol and enforces it on the model's output. Its core is a candidate $\times$ evidence matrix falsification protocol: every candidate is judged against every evidence item (support / oppose / irrelevant), at least half of the candidates must be falsified before a conclusion is licensed, and the surviving most-supported candidates are emitted in a machine-verifiable single-line format. Knowledge is externalized as declarative, domain-expert-revisable data and injected under the protocol's control---what knowledge is available at each stage is determined by a lifecycle state machine and information-sufficiency scoring---so knowledge serves the protocol rather than the model's free association, and knowledge correction is decoupled from both model weights and control flow.

We instantiate and evaluate the framework on TCM syndrome differentiation. This choice is deliberate rather than incidental. First, TCM differentiation is one of the few natural tasks that compels all three classical inference modes within a single decision loop~\cite{peirce1878}: \emph{abduction} (from four-examination signs to candidate syndromes), \emph{deduction} (from a syndrome to its treatment principle and prescription), and \emph{induction} (from case statistics to the differential clues of the data anchor). Second, its philosophical framework (yin--yang, five phases) acts as a built-in prior that constrains rather than replaces inference---a natural ally of protocol-constrained reasoning, since the domain itself already formalizes inference as constrained derivation under a world model. Third, abduction is the least mature inference mode in current AI: models that perform well on deduction still struggle with abduction~\cite{dougrezlewis2025abduction}, so a quantified, sign-conditional account on an abduction-dense, philosophy-constrained domain is stronger evidence of how protocol constraints interact with model discipline than a uniform gain on a deduction-dense one.

Both predictions of the hypothesis are confirmed (Section~\ref{sec:results}): the protocol rescues models whose zero-shot output discipline fails, stabilizes a high-variance model, and measurably costs already-disciplined models. Our contributions are:

\begin{enumerate}
  \item \textbf{A quantified boundary condition for constrained reasoning, established as controlled evidence.} Same-model comparisons across five base models---three repetitions, paired Wilcoxon tests with bootstrap confidence intervals, and explicit evaluation hygiene (Section~\ref{sec:hygiene})---show that the protocol's gain is sign-conditional on the base model's zero-shot output discipline: large significant gains where discipline fails (DeepSeek-V4-Flash $+0.267$/$+0.271$, robust under post-hoc budget neutralization on the certified rescue subset), an order-of-magnitude reduction in run-to-run standard deviation for a high-variance model (Kimi K3: 0.218 $\rightarrow$ 0.007), and significant costs where discipline is already internalized ($-0.087$/$-0.112$ on Qwen3-32B; $-0.116$/$-0.120$ on budget-matched Qwen3-4B).
  \item \textbf{The instrument: a model-agnostic agent framework for auditable decision making.} A five-stage causal-chain closed loop (hypothesis generation $\rightarrow$ verification $\rightarrow$ causal analysis $\rightarrow$ action $\rightarrow$ feedback verification) instantiated for TCM syndrome differentiation, enforced by a candidate $\times$ four-examination evidence matrix falsification protocol---three-state (support/oppose/irrelevant) cell judgments, row-level aggregation, a ``no fewer than half incompatible'' sufficiency rule, and a most-supported selection with a single-line structured conclusion---turning differentiation from one-shot association into a reviewable hypothesis test.
  \item \textbf{Mechanistic attribution of the gains and costs.} A rescue/terminating decomposition shows that the full-set gains on non-terminating models are entirely termination rescue, at a small significant cost on reliably terminating cases; format-only ablations identify the candidate-bounded format as the sign-conditional layer and reveal a sign-flipping complementarity between the causal chain and the format layer.
  \item \textbf{Evidence of protocol--knowledge decoupling.} Removing \emph{all} injected skill knowledge leaves benchmark scores unchanged while raising hit rate, showing that the protocol's measured effect is carried by structure rather than by knowledge content; a specialty-tuned knowledge file is further shown to contaminate broad-domain cases, validating specialty-bound knowledge routing.
  \item \textbf{Cross-domain transfer, operationalized rather than asserted.} A two-stage adaptation of the same protocol detects injected false premises in mathematical reasoning at 89.8\% recall with a 6.8\% false-positive rate and restores accuracy to the no-misinformation ceiling on a public external testbed---on a single, comparatively robust model with a small degradation headroom (0.039), so the ceiling result is a proof of capability rather than a robustness guarantee. Persuasion and belief-updating theories developed for human subjects---inoculation theory, the continued influence effect, and belief-revision theory---are used as \emph{interpretive lenses} for the protocol, each inherited in weakened, computable form and each offering a consistent interpretation of its behavior rather than an independently testable prediction (Sections~\ref{sec:related} and~\ref{sec:discussion}). We state the novelty of this instantiation as a scoped claim rather than a priority claim: the LLM--cognitive-science intersection is moving fast, and existing LLM--misinformation work we are aware of uses such theories to inoculate \emph{humans} or to \emph{evaluate} model vulnerability.
\end{enumerate}

The remainder of this paper is organized as follows. Section~\ref{sec:related} reviews related work. Section~\ref{sec:methods} presents the framework. Section~\ref{sec:experiments} describes the experimental setup. Section~\ref{sec:results} reports results. Section~\ref{sec:discussion} discusses findings and limitations, and Section~\ref{sec:conclusion} concludes.

\section{Related Work}
\label{sec:related}

\subsection{Constrained, verified, and safe generation}

Chain-of-thought (CoT) prompting~\cite{cot}, self-consistency~\cite{selfconsistency}, and verification-guided decoding improve reasoning reliability but lack domain-specific falsification structure: no candidate space, no evidence matrix, no sufficiency criterion. Retrieval-augmented generation (RAG)~\cite{rag} grounds generation in retrieved documents and mitigates hallucination, but retrieval returns associationally similar text and the downstream reasoning remains unconstrained; neuro-symbolic and GraphRAG-style frameworks applied to TCM knowledge retrieval~\cite{opentcm} constrain the knowledge source but likewise leave the reasoning path free. Our framework uses retrieval differently: for knowledge queries it bypasses generation entirely (direct original-text answers), while for decision tasks it injects declarative knowledge as constraints rather than as answer sources; and our matrix protocol adds a structured hypothesis-testing layer on top of these general techniques.

A complementary line of work treats reasoning itself as a security object. Recent work defines reasoning safety along logical consistency, computational efficiency, and manipulation resistance, and proposes runtime monitors that inspect each reasoning step and interrupt unsafe trajectories~\cite{reasoningsafety}. This ``detect-and-interrupt'' paradigm is complementary to ours: monitors discover bad reasoning post hoc and halt it, at the price of an additional model in the loop and non-trivial false-positive rates, whereas our protocol constrains generation at source---deterministically, without an auxiliary model---at the price of requiring tasks that can be structured into a candidate $\times$ evidence form. Our audit logs (per-cell judgments, protocol-compliance scores) provide the white-box evidence that monitor-based approaches can only approximate.

\subsection{Agentic and multi-agent medical systems}

ReAct-style agents interleave reasoning and acting within a single loop~\cite{react}, but the reasoning remains a free-form scratchpad without verifiable checkpoints. Beyond this single-agent paradigm, LLM-based agents have been explored for decision making along three lines. Multi-agent collaboration instantiates multi-disciplinary consultation as inter-agent dialogue: MedAgents~\cite{medagents} organizes expert recruitment, iterative discussion, and consensus formation for zero-shot medical reasoning, and MDAgents~\cite{mdagents} adaptively routes queries to a solo agent, a multi-disciplinary team, or an integrated care team according to estimated complexity. Environment simulation builds hospital-scale simulacra in which doctor agents evolve through practice (Agent Hospital~\cite{agenthospital}) or serves as interactive clinical evaluation environments (AgentClinic~\cite{agentclinic}). Task-oriented TCM agents are closest to our setting: JingFang~\cite{jingfang} decomposes TCM consultation into inquiry, syndrome-differentiation, and treatment agents with dual-stage formula retrieval; TCM-Agent~\cite{tcmagent} orchestrates multiple agents for network-pharmacology-driven herbal research; and CORE-Acu~\cite{coreacu} couples structured chain-of-thought with a safety knowledge graph whose symbolic veto mechanism blocks contraindicated acupuncture decisions.

These systems constrain the dialogue process or the tool layer, but none enforces a verifiable reasoning protocol within a single reasoning instance: conclusions emerge from agent dialogue or role orchestration rather than from a machine-checkable hypothesis-testing procedure, and multi-role dialogue multiplies inference cost. Our mechanism executes the full abduction--verification--deduction--action--feedback closed loop inside a single model instance with machine-checkable protocols, keeping inference cheap and auditable.

\subsection{Reasoning robustness under information contamination}
\label{sec:contamination-threat}

Beyond ordinary errors, LLM reasoning is vulnerable to contaminated inputs. Misinformation is one form of a broader threat: the manipulation of the reasoning process itself---of which injected false premises, persuasive pressure, and contaminated evidence are all instances. We frame the problem at this level of generality and evaluate on its best-documented instance. Prior work shows that misinformation propagates along reasoning chains: even when models possess correct internal knowledge, injected false premises reduce knowledgeable accuracy by 10--72\%~\cite{misinfoprop}, and explicit correction instructions succeed in fewer than half of cases. Reasoning post-training (\emph{thinking} models) does not eradicate the vulnerability, which is traced to the instruction-tuning stage rather than to reasoning capacity itself; the proposed remedy---fine-tuning on early factual correction---is training-time and model-bound~\cite{misinfoprop}. The same mechanism analysis shows that correction is most effective when applied \emph{early} in the reasoning trajectory, before the contaminated premise propagates downstream~\cite{misinfoprop}; this is precisely the ordering our two-stage design enforces at inference time---the verification stage (candidate $\times$ evidence matrix) always precedes the causal-analysis stage (Section~\ref{sec:methods}), so premises are checked before any derivation builds on them.

The failure pattern is not machine-specific: in human reasoning, misinformation continues to shape inference even after recipients explicitly accept a correction---the \emph{continued influence effect}~\cite{johnson1998updating,lewandowsky2012cie}, traced to memory failure~\cite{swirethompson2023cognition}. On the machine side, reasoning-optimized models resist majority pressure in social-influence simulations (conformity 3.13\% vs.\ up to 19.45\%)~\cite{socialinfluence}, and enabling thinking mode reduces persuasion rates in multi-agent persuasion experiments~\cite{persuasionduality}. Yet such protection is bound to training, offers no auditability, and does not distinguish right from wrong: the thinking-mode benefit is documented precisely when the initial answer was already correct~\cite{persuasionduality}, so on initially wrong answers internalized resistance entrenches error rather than preventing it.

Our falsification protocol is an externally enforced, evidence-checking structure of this kind: a candidate survives only when supporting evidence outnumbers opposing evidence, and the re-examination loop re-opens the decision whenever falsification falls short. As an inference-time, training-free layer it complements training-time corrections and transfers across models; and because resistance comes from evidence checking rather than stance retention, it yields \emph{correctable insistence}, avoiding the error-entrenchment failure mode of stance-based defenses~\cite{persuasionduality,socialinfluence}. Direct evaluation under adversarial premise injection is reported in Section~\ref{sec:misinfoexp}: the mechanism transfers as a \emph{two-stage adaptation}---the candidate space becomes the premise equations, the evidence columns become numeric/relation/self-consistency checks, and falsified premises are rewritten reminder-style before derivation---rather than as a verbatim reuse of the clinical pipeline.

Two persuasion-theoretic lenses inform the design and its boundary condition. Inoculation theory~\cite{mcguire1964inoculation} confers resistance through forewarning plus refutational preemption of weakened attacks (established meta-analytically~\cite{banas2010inoculation}), and technique-based prebunking transfers across contents when what is neutralized is a fallacy \emph{type} rather than a topic~\cite{vanderlinden2017science,lewandowsky2021inoculation,banas2013logic,biddlestone2023conjunction}; requiring refutation practice on every candidate before any conclusion is licensed is enforced refutational preemption. The Elaboration Likelihood Model~\cite{petty1986elm} distinguishes central-route scrutiny from peripheral, cue-driven processing, and zero-shot LLM reasoning is biased toward the peripheral route~\cite{ziabari2025spectrum}; under this lens the matrix protocol is an \textbf{elaboration enforcer}, consistent with the output-discipline boundary condition of Section~\ref{sec:significance}. As far as our literature coverage extends, existing LLM--misinformation work treats models as \emph{generators} of prebunking material or as \emph{targets} whose vulnerability is measured; neither treats the model's own reasoning protocol as the \emph{carrier} of an inoculation structure. Section~\ref{sec:discussion} develops both mappings.

\subsection{LLMs and benchmarks for TCM}

Several benchmarks evaluate TCM capabilities of LLMs. TCMBench offers knowledge-centric question answering built on the TCM licensing examination, and reports that no evaluated model---including GPT-4---reaches the 60\% passing line~\cite{tcmbench}. TCM-SD evaluates syndrome derivation from case analysis over 54,152 real-world clinical records covering 148 syndromes~\cite{tcmsd}. MTCMB~\cite{mtcmb} extends evaluation to a multi-task setting spanning knowledge, reasoning, and safety, including diagnosis sub-tasks on syndrome differentiation and prescription recommendation. TCMEval-SDT~\cite{tcmeval} is, to our knowledge, the first benchmark targeting the thought process of syndrome differentiation, providing 300 expert-annotated cases with multi-select options for pathogenesis and syndrome, and zero-shot scores for ChatGPT, Gemini 1.5 Pro, ChatGLM-130B, and Tongyi Qianwen. We adopt TCMEval-SDT and its official scoring script for comparability.

On the model side, domain-adapted Chinese medical and TCM LLMs improve domain knowledge via continued pre-training and fine-tuning: Zhongjing~\cite{zhongjing} combines a real multi-turn dialogue corpus with reinforcement learning from human feedback (RLHF) over expert-ranked outputs; HuatuoGPT-II~\cite{huatuogpt2} proposes a single-stage domain-adaptation protocol; and TCM-oriented models such as TCMChat, BianCang, and Qibo~\cite{tcmchat,biancang,qibo} target syndrome differentiation, formula recommendation, and knowledge QA respectively. Their reasoning, however, remains a single-pass free-form generation that cannot be audited or falsified step by step. Our framework is orthogonal to these efforts: it wraps an arbitrary base model---without any parameter updates---in a verifiable reasoning protocol. Consistently, our experiments show that TCM post-training changes behavior little relative to the base model ($\pm0.01$), while the protocol's cost replicates on the post-trained variant ($-0.111$/$-0.112$)---i.e., training-time knowledge injection and inference-time protocol constraints address different bottlenecks.

\section{The Framework}
\label{sec:methods}

\subsection{Overview: a five-stage causal-chain closed loop}

The framework formalizes differentiation as a closed loop of five stages:

\begin{enumerate}
  \item \textbf{Hypothesis generation (abduction).} Given patient evidence, a candidate syndrome space is constructed from (i) candidates provided by the user or upstream reasoning (marker-delimited), (ii) specialty skill routing, and (iii) knowledge-base syndrome retrieval. A scope constraint restricts all subsequent conclusions to this space.
  \item \textbf{Hypothesis verification (evidence matrix).} For each candidate, authoritative diagnostic standards (chief/accompanying symptoms, tongue, pulse, treatment principle) are extracted from declarative knowledge files---with alias mapping that unifies synonymous syndrome names (e.g., ``phlegm-heat harassing the heart'' and ``phlegm-heat internal disturbance'') and a per-syndrome length cap---and a patient-evidence checklist is built with positive/negative feature rules (longest-match-first, negation-priority), yielding a candidate $\times$ evidence matrix.
  \item \textbf{Causal analysis (constrained reasoning chain).} The model is required to reason along a forced chain---four-examination analysis $\rightarrow$ eight principles $\rightarrow$ disease nature $\rightarrow$ disease location $\rightarrow$ syndrome synthesis---explicitly forbidden to jump from a single sign to a syndrome.
  \item \textbf{Action.} The treatment principle is derived over the TCM logic system; formula names, herbs, doses, and acupoints must come from knowledge-base originals and are verified via structured tool calls.
  \item \textbf{Feedback verification.} Prescriptions pass hard safety checks (incompatibility, pregnancy, toxicity dosage); diagnosis consistency across turns is enforced via conclusion caching and drift detection; reasoning quality is scored for audit.
\end{enumerate}

Figure~\ref{fig:architecture} shows the four-layer system architecture; Figure~\ref{fig:pipeline} shows the end-to-end processing pipeline.

\begin{figure*}[t]
\centering
\includegraphics[width=\textwidth]{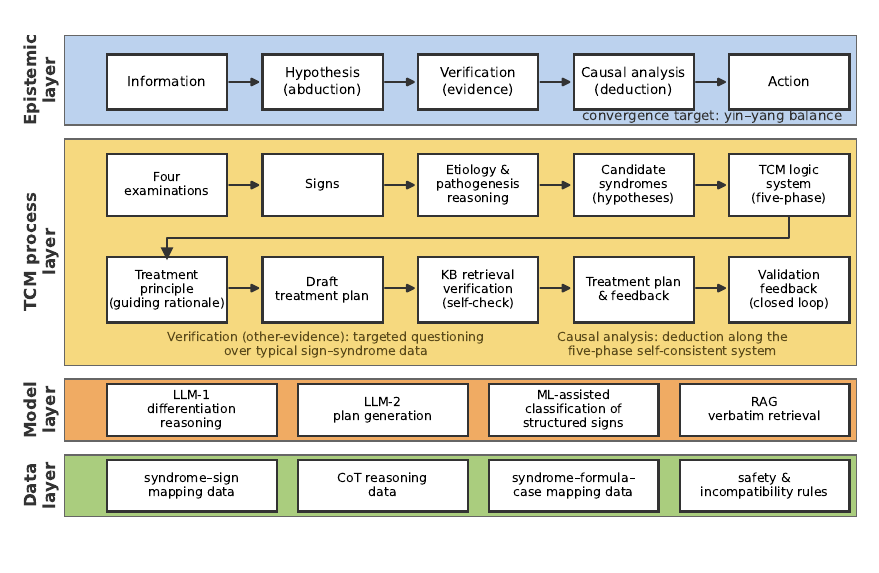}
\caption{Four-layer architecture. The epistemic layer defines the hypothesis--verification--causal-analysis--action loop; the TCM process layer carries the differentiation workflow, closing the loop through the validation-feedback stage; the model layer (LLM-1/LLM-2/ML/RAG) executes it; the data layer holds syndrome--sign mappings, reasoning data, case mappings, and safety rules. Within the process layer, verification takes the form of targeted questioning over typical sign--syndrome data, and causal analysis proceeds over the five-phase self-consistent system.}
\label{fig:architecture}
\end{figure*}

\begin{figure*}[t]
\centering
\includegraphics[width=\textwidth]{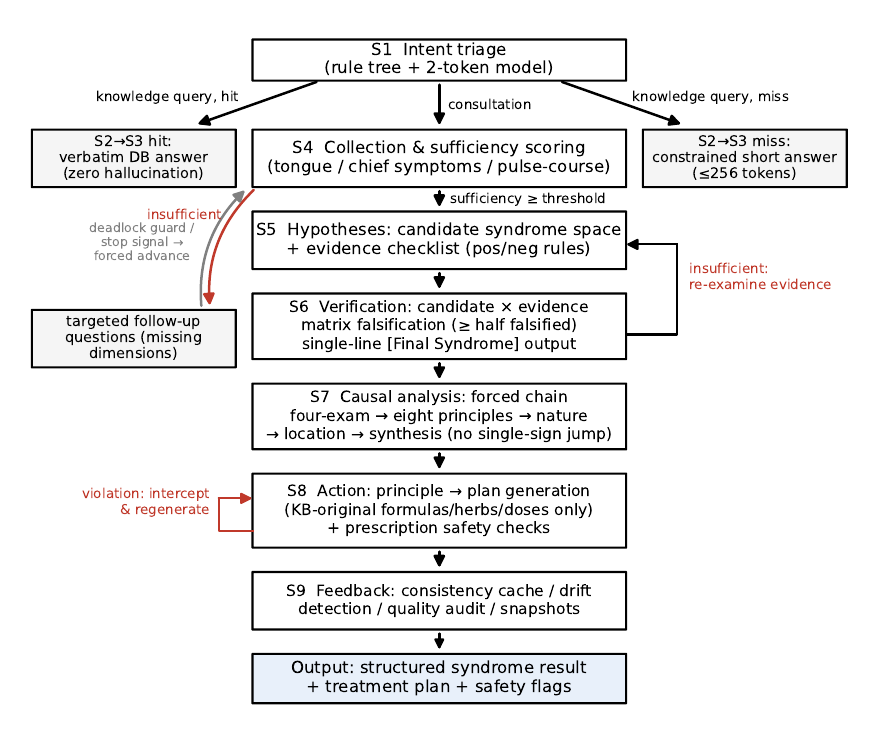}
\caption{End-to-end processing pipeline (S1--S9) with anti-deadlock questioning guards, the insufficiency re-examination loop after matrix falsification, and prescription-violation interception.}
\label{fig:pipeline}
\end{figure*}

\subsection{Candidate $\times$ evidence matrix falsification protocol}
\label{sec:matrix}

The protocol, the core of this work, constrains the model's conclusion process as follows:

\begin{itemize}
  \item \textbf{Matrix construction.} Rows are candidate syndromes; columns are the patient's four-examination evidence items.
  \item \textbf{Cell judgment.} Each cell is judged \emph{support}, \emph{oppose}, or \emph{irrelevant}, justified by the injected diagnostic standards.
  \item \textbf{Row aggregation.} A candidate is \emph{compatible} iff supporting evidence outnumbers opposing evidence and key evidence is supported; otherwise it is \emph{incompatible}.
  \item \textbf{Sufficiency rule (falsification).} The number of incompatible candidates must be \textbf{at least half of the candidate space}; otherwise differentiation is deemed insufficient and evidence must be re-examined.
  \item \textbf{Selection and output.} The final conclusion is the compatible candidate with the most supporting evidence (or several, in multi-answer settings), emitted on a single line in a fixed marker format (\texttt{[Final Syndrome]: X})---machine-extractable and checkable.
  \item \textbf{Prohibitions.} Choosing by candidate order and copying example conclusions are explicitly banned.
\item \textbf{Explicit refutation.} An opposed cell judgment must name the contradicting evidence (e.g., ``opposed by E3''). Silent elimination of a candidate is discouraged: memory-updating research shows that corrections which explicitly re-state the discredited item---co-activating it with the conflicting evidence---update beliefs more effectively than corrections that only supply the correct answer~\cite{wahlheim2020reminders}, and the logged evidence pointer makes each falsification auditable.
\end{itemize}

Every cell judgment can be cross-checked against the evidence checklist and diagnostic standards, making the reasoning white-box and auditable. Figure~\ref{fig:matrix} illustrates the protocol on a worked three-candidate example. Formal readings of the protocol are collected in \ref{sec:agm-reading}: a belief-revision (AGM) reading of the aggregation, contraction, and consistency rules; the proof of Proposition~\ref{prop:accumulation} with its caveats.

\paragraph{Statistical reading: sequential multi-hypothesis testing} Treat the candidate set $\mathcal{H}=\{h_1,\dots,h_m\}$ as $m$ competing hypotheses and each evidence item $e_j$ as an observation; the cell function $f(h_i,e_j)\in\{+1,0,-1\}$ is a per-hypothesis score. Define the row statistic
\begin{equation}
S(h_i) \;=\; \sum_{j=1}^{n} w_j\, f(h_i,e_j), \qquad w_j \ge 0,
\label{eq:rowstat}
\end{equation}
with compatibility declared iff $S(h_i)>0$ and at least one key evidence item $e_{\mathrm{key}}$ satisfies $f=+1$. The deployed protocol uses uniform weights ($w_j{=}1$): the model counts supporting against opposing items, and the key-evidence requirement enters as a separate hard gate rather than as a large weight. The ``at least half falsified'' rule is then a \emph{sufficiency stopping criterion}: with the falsified set containing both score-incompatible and key-evidence-failed candidates, inference is licensed only once the condition of Eq.~\eqref{eq:suffrule} holds
\begin{equation}
\frac{\left|\{\,h_i : S(h_i)\le 0 \ \text{or}\ f(h_i,e_{\mathrm{key}})\ne +1\,\}\right|}{m} \;\ge\; \tfrac{1}{2}.
\label{eq:suffrule}
\end{equation}
Under a \emph{mutually exclusive syndromes, conditionally independent evidence} model, each falsified candidate transfers posterior probability to the survivors, so the posterior concentration on the remaining hypothesis set increases monotonically in the number of eliminated candidates. The following proposition makes the intuition ``more evidence, more reliable conclusions'' a theorem rather than a slogan.

\begin{proposition}[Evidence-accumulation bound]
\label{prop:accumulation}
Assume (i) the true syndrome $h^{\star}\in\mathcal{H}$ is present; (ii) cell judgments are conditionally independent given $h^{\star}$; (iii) each cell judgment is correct with probability at least $1-\epsilon$, $\epsilon<\tfrac{1}{2}$, where correctness is judged against the ground-truth cell value $v(h,e)\in\{+1,0,-1\}$, and errors are \emph{sign-symmetric} in the sense that $\mathbb{E}[f(h,e)]=(1-2\epsilon)\,v(h,e)$; (iv) key-evidence requirements hold; and (v) for every spurious candidate $h_i$, all $n$ evidence items weakly favor the true candidate ($v(h^{\star},e_j)\ge v(h_i,e_j)$ for all $j$), with at least $n_i$ of them strictly discriminating ($v(h^{\star},e_j){=}+1$, $v(h_i,e_j){=}-1$). Then the probability that the protocol's selection rule returns $h^{\star}$ is bounded below by
\begin{equation}
\Pr[\,\widehat{h}=h^{\star}\,] \;\ge\; 1 - \sum_{i\ne\star}\exp\!\Bigl(-2(\tfrac12-\epsilon)^{2}\,\frac{n_i^{2}}{n}\Bigr),
\label{eq:hoeffding}\end{equation}
a Hoeffding-type bound~\cite{hoeffding1963}. In the special case of perfect discrimination on every item ($n_i{=}n$ for all $i$), Eq.~\eqref{eq:hoeffding} reduces to $1-(m-1)\exp(-2(\tfrac12-\epsilon)^2 n)$, non-trivial once $n$ exceeds $\ln(m{-}1)/(2(\tfrac12-\epsilon)^2)$.
\end{proposition}

\paragraph{Information-theoretic and decision-theoretic readings of the one-half threshold} Why one half? Two readings illuminate the deployed choice, and we state upfront what neither does: uniquely select the threshold. \emph{(i)} An information reading gives the symmetric point a label---under a uniform prior, falsifying half the candidates removes exactly one bit of hypothesis-space uncertainty (Eq.~\eqref{eq:entropy})---but the marginal information of the $(k{+}1)$-th falsification, $\log_{2}\!\frac{m-k}{m-k-1}$, \emph{increases} with $k$, so the same calculus pushes toward stopping as late as possible: a landmark, not an optimum. \emph{(ii)} A decision-theoretic reading defines the conclusion confidence $C(k)=k/m$ (Eq.~\eqref{eq:confidence}): raising the required $C$ monotonically shrinks both the licensed output set and the selection error, a trade that holds for \emph{any} threshold and therefore cannot single out one half, but that makes $C$ a per-case, matrix-computable confidence score---the same budget logic as uncertainty sampling in active learning~\cite{lewis1994uncertainty,settles2009active}, applied at inference time. What selects the operating point is therefore empirical: the $\kappa\in\{\tfrac14,\tfrac12,\tfrac34\}$ ablation of Section~\ref{sec:threshold} (Table~\ref{tab:kappa}), which shows the deployed one half is \emph{not} the winner and that the profitable direction is stricter falsification. At the far edge, when falsification drives $k$ toward $m$, the conclusion set may empty, so the protocol's re-examination loop, not further elimination, is the correct response, and $C$ should be read jointly with the size of the surviving set.
\begin{equation}
\Delta H(k) \;=\; \log_{2}\frac{m}{\,m-k\,}\ \text{bits}, \qquad \Delta H\bigl(\tfrac{m}{2}\bigr)=1\ \text{bit},
\label{eq:entropy}
\end{equation}
\begin{equation}
C(k) \;=\; 1-\frac{|\mathcal{H}_{\mathrm{surv}}|}{m} \;=\; \frac{k}{m}.
\label{eq:confidence}
\end{equation}

\begin{figure*}[t]
\centering
\includegraphics[width=\textwidth]{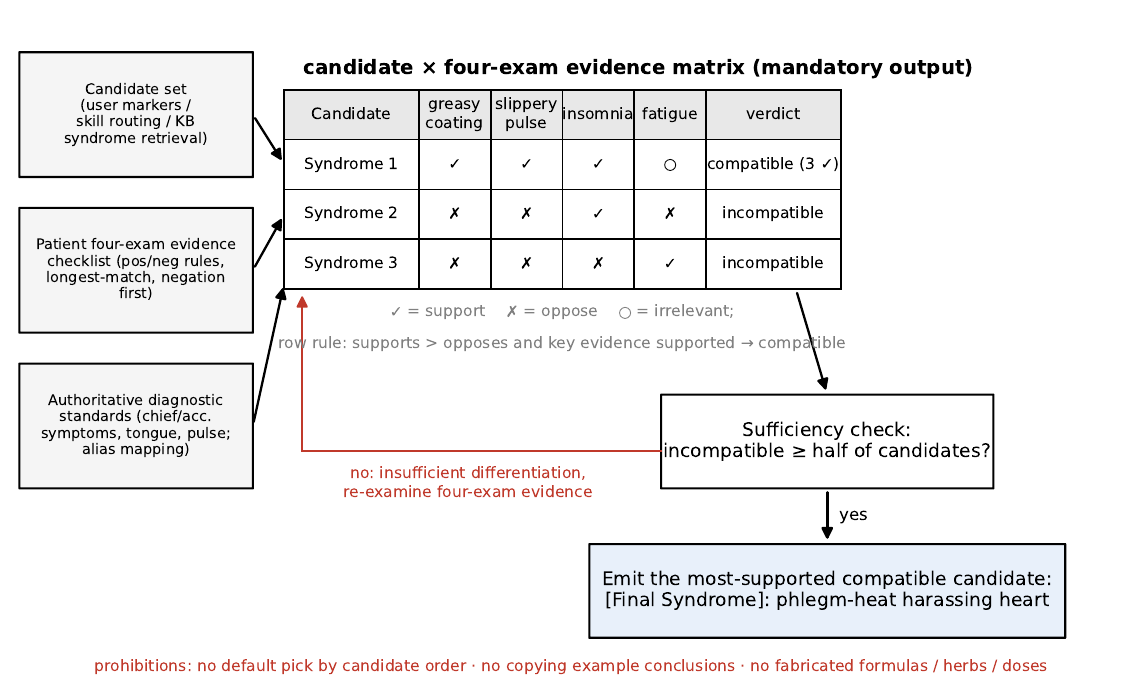}
\caption{The candidate $\times$ four-examination evidence matrix falsification protocol on a worked example. Three knowledge inputs build the matrix; row-level aggregation yields compatibility verdicts; the sufficiency check requires at least half of the candidates to be incompatible, otherwise evidence is re-examined; the most-supported compatible candidate is emitted in the single-line protocol format.}
\label{fig:matrix}
\end{figure*}

\subsection{Knowledge assets and dual-anchor injection}
\label{sec:knowledge}

All medical knowledge lives outside the code as data, in three forms. \textbf{Case database.} A local SQLite database (\texttt{tcm.db}) holds 8,147 case records with \emph{question/chain-of-thought/response/syndrome/formula} fields, aggregated from 5,903 TCM question--answer pairs, 738 acupuncture case records from a renowned physician's lectures, and 1,506 licensed-physician examination questions. The database additionally contains 58 structured classical formulas, 23 incompatibility pairs (the eighteen incompatibilities and nineteen mutual antagonisms canon enforced at prescription time, plus one extended pair), and 17 pregnancy-contraindication entries (14 banned and 3 caution-level herbs). A Chinese-bigram FTS5 full-text index supports the retrieval cascade of \ref{sec:retrieval}. \textbf{Distilled mapping datasets.} From 1,103 TCM question--answer pairs we distilled syndrome$\leftrightarrow$sign, formula-symptom$\leftrightarrow$syndrome, and prescription mappings. A separate coded dataset of 539 insomnia cases with tongue-sign inputs serves as the specialty pilot from which the data-driven differential clues below are distilled; other specialties follow the same procedure. \textbf{Declarative skill files.} Textbook standards (chiefly \emph{TCM Diagnostics}~\cite{tcmtextbook}) and the above datasets are refined into 23 Markdown skill files organized into six specialty skills (basic diagnosis, insomnia, hypertension, gynecology, pediatrics, acupuncture). Skills are routed by specialty priority and injected per lifecycle stage, so physicians can revise knowledge without touching code.

\label{sec:dualanchor}
At differentiation time, two complementary anchors are injected alongside the forced causal chain. \textbf{Theory anchor (methodological rules).} Key decision rules distilled from textbooks, e.g., \emph{a thick coating alone cannot exclude deficiency syndromes}; \emph{the absence of a greasy coating alone cannot exclude phlegm-heat}; \emph{a pale-red tongue alone cannot exclude liver-fire or phlegm-heat}; \emph{a pale-purple tongue alone does not diagnose blood stasis}; and a \emph{symptom-priority principle} when signs conflict. \textbf{Data anchor (statistical clues).} Conditional probabilities distilled from the 539-case specialty dataset, e.g., among white-coating samples about 47\% are phlegm-heat harassing the heart; teeth-marked tongues occur in 72\% of heart--gallbladder qi deficiency and 54\% of heart--spleen deficiency cases; greasy coating in 66\% of phlegm-heat cases. Syndrome prior distributions and explicit data-bias warnings are injected together, so that when textbook theory and sample statistics conflict the model still has a decision basis. \ref{sec:bayes-reading} gives a Bayesian reading of the two anchors: the data anchor supplies posterior fragments rather than likelihoods, the theory anchor acts as a constraint set rather than a prior, and naive-Bayes pooling links both to the matrix score of Section~\ref{sec:matrix}.

Injection budgets scale with model size: 12,000--18,000 characters for 4B models and 28,000--32,000 characters for 32B models. The model runs with thinking (reasoning) mode enabled, and reasoning blocks are stored separately from final answers for audit.

\subsection{Intent triage, information sufficiency, and active questioning}
\label{sec:sufficiency}

Incoming input first passes a \textbf{two-level intent triage} (a rule decision tree, then a two-token model classification into consultation / closing / off-topic), so that knowledge queries never enter the differentiation pipeline. Knowledge queries are served by a \textbf{three-level retrieval cascade} over the case database, and the top record's answer text is returned \emph{verbatim from the database without generation}---zero hallucination, at millisecond latency; only when retrieval fails does the model answer within a restricted budget, still outside the differentiation pipeline. \ref{sec:retrieval} details the triage rules, the cascade, the scoring rule, and a worked example.

Each consultation turn is scored for information sufficiency over the evidence categories collected so far (tongue signs, chief symptoms, pulse/course); reaching the sufficiency threshold advances directly to differentiation, and below threshold the agent names the missing dimensions and asks targeted follow-up questions. Anti-deadlock guards truncate repeated question-only outputs and honor explicit user stop signals, preventing infinite questioning. Session-level completeness is estimated as the coverage of chief complaint, symptoms, tongue, and pulse, and gates the COLLECT-exit transition of Section~\ref{sec:fsm}. \ref{sec:ig-reading} develops the information-theoretic reading of this rule---a greedy information-gain policy in disguise---together with its upgrade path.

\subsection{Lifecycle state machine and safety invariants}
\label{sec:fsm}

The consultation is modeled as a finite state machine with nine states: INIT, COLLECT, DIAGNOSE, PLAN, PRESCRIBE, REVIEW, FOLLOWUP, COMPLETED, and ABORTED. Transitions carry preconditions (symptom completeness $\geq 0.6$ to leave COLLECT; diagnosis confidence $\geq 0.7$ to leave DIAGNOSE); advancement decisions fuse token-level and semantic confidence with weights 0.3/0.7, with a turn-count fallback; illegal transitions are rejected outright and timeout guards prevent stalls. Table~\ref{tab:mapping} maps the method stages onto lifecycle states.

The nine-state machine is, formally, a finite automaton $(Q,\Sigma,\delta,q_{0},F)$ whose transition preconditions are guards, and this affords the framework a specification language rather than only an implementation. Illegal-transition rejection is a \emph{safety invariant}: in linear temporal logic~\cite{pnueli1977ltl}, the deployed checks enforce
\begin{align}
&\Box\,\bigl(\textsc{prescribe} \;\Rightarrow\; \Diamond\,(\textsc{review} \vee \textsc{aborted})\bigr),\label{eq:ltl1}\\
&\Box\,\bigl((\textsc{collect} \;\wedge\; \text{completeness} < 0.6 \;\wedge\; \neg\mathit{forced}) \;\Rightarrow\; \bigcirc\,\textsc{collect}\bigr),\label{eq:ltl2}
\end{align}
where $\Box$ is ``globally'', $\bigcirc$ ``next'', $\Diamond$ ``eventually'', and $\mathit{forced}$ marks the forced-advance paths of Section~\ref{sec:sufficiency} (turn-count fallback, deadlock truncation, timeout guards). The first invariant requires every \textsc{prescribe} to be eventually followed by a \textsc{review} or an abort, aligned with the PLAN$\rightarrow$PRESCRIBE$\rightarrow$REVIEW order of Table~\ref{tab:mapping}; the second is the COLLECT-exit precondition restricted to non-forced trajectories (stop-signal advancement needs no exclusion because it first raises completeness to the threshold). Both are enforced at runtime by construction and are checkable \emph{post hoc} on any logged trajectory, which is what makes the audit layer (Section~\ref{sec:audit}) a verification step rather than a log dump. Because the automaton is finite and the guards are decidable predicates, these properties are in principle model-checkable over the implementation's transition graph; we leave mechanized model checking to future work.

\begin{table}[t]
\centering
\caption{Mapping between method stages and lifecycle states.}
\label{tab:mapping}
\small
\resizebox{\columnwidth}{!}{%
\begin{tabular}{@{}ll@{}}
\toprule
Method stage & Lifecycle state(s) \\
\midrule
Intent triage, direct retrieval & outside the machine \\
Information collection, sufficiency & COLLECT \\
Hypothesis \& evidence construction & DIAGNOSE (prompt injection) \\
Matrix falsification \& output & DIAGNOSE (output protocol) \\
Causal-chain reasoning & DIAGNOSE (reasoning) \\
Action: plan, prescribe, review & PLAN $\rightarrow$ PRESCRIBE $\rightarrow$ REVIEW \\
Feedback \& consistency & FOLLOWUP $\rightarrow$ COMPLETED \\
Runtime protection \& audit & cross-cutting \\
Hard violations, fatal signals & ABORTED \\
\bottomrule
\end{tabular}%
}
\end{table}

\subsection{Consistency, feedback, recovery, and runtime robustness}
\label{sec:consistency}

\textbf{Consistency.} Confirmed syndromes are cached; a follow-up turn (shorter than 50 characters and carrying no new tongue/coating/pulse/symptom/course terms) triggers a full prompt rewrite---\emph{``syndrome X is confirmed; explain its rationale and differentiate it from the syndromes the user mentioned; do not re-diagnose''}---rather than a fresh differentiation, keeping conclusions stable across turns. Within a turn, multiple distinct syndrome conclusions trigger drift detection, and the last conclusion stands. Reasoning blocks are scored over six dimensions (four-examination flow, eight-principle differentiation, tongue diagnosis, pulse diagnosis, candidate comparison, TCM terminology; score $=$ hit classes$/6$) and logged for audit. Under stationary inputs the cache keeps the conclusion fixed; a genuine change of conclusion requires new evidence that pushes re-diagnosis across a decision boundary, and under the matrix protocol this evolution is monotone in the accumulated evidence---adding evidence refines the same logged matrix rather than restarting it.

\textbf{Feedback and recovery.} Runtime signals (tool failures, safety violations, etc.) are graded into four severities---info, warning, severe, fatal---mapped to six responses (ignore, warn, adjust, pause, roll back, abort). Hard-constraint signals are handled synchronously (target latency 50 ms), and fatal signals drive the state machine into ABORTED. Snapshots are saved automatically on every state change (SHA-256 integrity, last 10 versions retained); failures follow a three-level policy---automatic retry within 5 s, rollback to the latest safe snapshot within 30 s, or manual intervention.

\textbf{Runtime robustness.} Four cross-cutting subsystems keep long consultations stable on small models: layered context budgeting across permanent/session/task windows; streaming guardrails (hard token caps, repetition-loop and language-mixing detection); five-dimension health monitoring; and trajectory logging with privacy de-identification. \ref{sec:runtime} details the budgets and detection rules.

\subsection{Prescription safety hard constraints}
\label{sec:safety}

Before any prescription is emitted, layered mandatory checks run: set-intersection detection over the eighteen incompatibilities (14 pairs) and nineteen mutual antagonisms (8 pairs); pharmacopoeia dose ceilings for 15 toxic herbs~\cite{pharmacopoeia}; 14 pregnancy-banned herbs when the patient is pregnant; human-review flags for children under 14 and adults over 70; and outright refusal of non-TCM operations such as prescribing Western medication or recommending surgery. Violating prescriptions are intercepted and regenerated, never silently patched. During plan and prescription generation the model may invoke nine structured tools through seven textual markers; every call passes a four-stage audited pipeline (pre-call audit; sandboxed execution with timeout and retries; post-call output validation; audit logging), detailed in \ref{sec:tools}.

\subsection{Model-agnostic adapter and deployment}

A unified adapter (local Transformers / vLLM cluster / OpenAI-compatible APIs / mock) decouples the mechanism from any specific model, with automatic device detection, thinking/reasoning separation, logprob confidence, and parameter auto-scaling of knowledge-injection budgets across model sizes.

\section{Experiments}
\label{sec:experiments}

\subsection{Dataset}

We evaluate on \textbf{TCMEval-SDT}~\cite{tcmeval} (CC-BY 4.0), the public benchmark for TCM syndrome-differentiation reasoning. It contains 300 expert-annotated cases; only the 200-case training split includes gold answers (multi-select among 10 options A--J for both pathogenesis and syndrome), and is therefore used for evaluation. Gold syndrome sets contain 1--4 items and pathogenesis sets 1--6 items.

\subsection{Evaluation protocol}

Options are injected as single-quoted candidate markers into the agent input, triggering the matrix protocol; textual conclusions are extracted and mapped back to option letters. We report: the \textbf{official-script score} (the benchmark's \texttt{score\_proportional}: $|A \cap B| / (|A| + |B \setminus A|)$ per case, averaged); \textbf{hit rate} (fraction of cases with non-empty prediction $\cap$ gold); \textbf{exact match}; and \textbf{macro-F1} (Tables~\ref{tab:main} and~\ref{tab:aux}).

\paragraph{Component activation under the benchmark pathway} We state explicitly which framework components the benchmark exercises, since candidate injection bypasses parts of the loop. \emph{Activated and measured}: declarative knowledge injection and skill routing (Section~\ref{sec:knowledge}), the forced causal-chain constraint (Section~\ref{sec:methods}), the matrix falsification protocol with its sufficiency rule, and the single-line output protocol. \emph{Bypassed by design}: abductive candidate construction (the benchmark supplies the candidate set---a deliberate isolation of the verification stage); active questioning and sufficiency scoring (inputs are complete case narratives, single-turn); the action-stage tool layer and prescription safety checks (the benchmark tasks end at differentiation, before prescription); and the multi-turn feedback loop (each case runs in a fresh context). The headline numbers therefore measure the \emph{verification and causal-analysis core} of the loop, not the full five-stage pipeline. The remaining components are validated elsewhere: active questioning and mechanism execution in the 539-case pilot (Section~\ref{sec:audit}), and the safety layer, state machine, tool pipeline, and recovery in the 21-test regression suite (Section~\ref{sec:audit}).

\subsection{Evaluation hygiene}
\label{sec:hygiene}

Three declarations address evaluation-contamination concerns. \emph{(i) Knowledge--benchmark overlap.} We computed an 8-gram containment check between each of the 200 benchmark cases (Clinical Data + Clinical Information) and every knowledge text the agent can retrieve (8,147 \texttt{tcm.db} question fields and all 31 skill/knowledge Markdown files, 8,178 texts total): the maximum per-case containment is 1.5\% (mean 0.1\%), i.e., no benchmark case shares substantive text with the knowledge base. \emph{(ii) Knowledge-iteration history.} All knowledge and skill files were frozen before benchmark contact. The two documented iterations of the differential data table (v0.1$\rightarrow$v0.2) and the differential guide (v0.2$\rightarrow$v0.3) were driven exclusively by the 539-case insomnia pilot (a private dataset disjoint from the benchmark) and completed before the first benchmark campaign began; post-hoc benchmark failure analyses in this paper were used for interpretation only and were never fed back into knowledge, prompts, or thresholds. Consistently, the ablation evidence runs opposite to any data-snooping hypothesis: the insomnia-tuned guide \emph{hurt} benchmark performance (Section~\ref{sec:contamination})---contamination was detected and reported, not exploited. \emph{(iii) Hyperparameters.} All protocol thresholds (completeness 0.6, diagnosis confidence 0.7, injection budgets) were fixed at development time and never tuned on the benchmark; the only benchmark-specific adaptations are the disclosed evaluation-mode output protocol (multi-answer \texttt{[Final Syndrome]} format) and the locked base-skill routing.

\subsection{Statistical analysis}
\label{sec:stats}

To separate genuine effects from decoding variance, every configuration that supports sampling is evaluated in \textbf{three independent repetitions} over the same 200 cases per task. We use \emph{run} and \emph{repetition} interchangeably, and \emph{case-run} to denote one (case, repetition) pair. We report run-level mean $\pm$ std, and test paired comparisons on per-case scores averaged across repetitions (unit of analysis: the case, $n{=}200$) using the two-sided \textbf{Wilcoxon signed-rank test}, complemented by a \textbf{paired bootstrap} 95\% confidence interval of the mean difference (10,000 resamples, fixed seed). Raw $p$-values are reported throughout. For confirmatory interpretation we apply a Bonferroni correction within the same-model mechanism-gain family (5 models $\times$ 2 tasks $=10$ tests, $\alpha{=}0.005$); cross-model comparisons mix repetition counts and configurations, so they are reported \emph{descriptively, without $p$-values}.

\textbf{Two inferential levels.} Case-level tests average over repetitions and answer the case-population question, but they can understate uncertainty where run-level variance is material (for K3 zero-shot, run means span 0.245--0.636). We therefore additionally report a conservative \textbf{run-level} check---Welch's $t$-test over the three run means per side (unit: the run, $n{=}3$)---for every three-repetition comparison (Table~\ref{tab:sig}). The two levels answer different questions and we report both; where they disagree we say so explicitly. Run-level inference at $n{=}3$ is low-powered by construction, so a non-significant run-level $p$ is weak evidence rather than evidence of absence. Variance-ratio claims are reported with 95\% CIs (variance-ratio $F$-test, df 2,2). We write ``n.s.''\ for non-significant, and reserve ``significant'' for $p$ below the stated threshold.

\subsection{Configurations}

\begin{itemize}
  \item \textbf{Qwen3-4B / Qwen3-32B}~\cite{qwen3}: zero-shot direct answering, and with the full framework (local vLLM);
  \item \textbf{PT-4B + agent}: a TCM post-trained 4B model with the full framework;
  \item \textbf{DeepSeek-V4-Flash}~\cite{deepseek}: zero-shot and with the framework (official API);
  \item \textbf{DeepSeek-V4-Pro}: zero-shot and with the framework (official API);
  \item \textbf{K3 zero-shot / K3 + agent}: Kimi K3 (2026 release)~\cite{kimi} without and with the framework.
\end{itemize}

Paper baselines (Gemini 1.5 Pro $\approx$0.475 syndrome, ChatGLM-130B $\approx$0.532 pathogenesis, converted from the paper's weighted sums to per-case means) are shown for reference; note they are 2023--2024-era zero-shot evaluations on a random 50-case subset.

\subsection{Ablation design}

Starting from the full agent (\textbf{G0}, Qwen3-4B), we remove one component per group: \textbf{G1} no matrix-and-format instruction block (the falsification-matrix requirement, the candidate-bounded single-line emission mandate, and the associated prohibitions are one prompt block and are removed jointly; knowledge injection is retained); \textbf{G2} no evidence checklist/authoritative standards; \textbf{G3} no differential-diagnosis guide (insomnia-tuned); \textbf{G4} no skill-knowledge injection at all. Each group is re-evaluated on both tasks (200 cases each, three repetitions). A separate \emph{format-only} baseline---candidate markers and the mandatory single-line output format, but no matrix and no knowledge injection---separates the format layer from the content layer that G1 removes jointly, and is run on both the budget-matched Qwen3-4B and DeepSeek-V4-Flash. On Qwen3-4B the factorial is completed by two further arms: a \emph{no-matrix} arm (\textbf{G5}, content + format, no matrix requirement) and a \emph{no-format} arm (\textbf{G6}, content + matrix, no emission mandate).

\subsection{Implementation details}

Local/server runs use Qwen3-4B (fp16, thinking enabled, max 2048 new tokens) and Qwen3-32B (vLLM, thinking enabled, max 4096 new tokens), both from unchanged local checkpoints (Qwen3 release of 2025); the DeepSeek runs use the official API. The K3 run uses the Kimi coding API with temperature fixed at 1.0 (the only value K3 accepts). All evaluations resume from checkpoints and log every case's full reasoning for audit. On budget symmetry: zero-shot baselines follow the benchmark paper's protocol at an 8,192-token budget, while agent runs receive the larger budgets the protocol's matrix output requires; the resulting asymmetry means full-set gains contain termination rescue, which is why the truncation-free decomposition of Section~\ref{sec:significance}---not the raw full-set contrast---is our primary controlled evidence. To control a temperature confound in the Qwen3-32B comparison (zero-shot at temperature 0 vs.\ agent at 0.6), the 32B agent was additionally evaluated at temperature 0. For unambiguous cross-reference, every reported number is tied to a named campaign: \textbf{C1} = initial runs (single repetition, temperature 0.6--0.7, 16,384-token budget, \texttt{/no\_think} for flash); \textbf{C2} = repetition campaign (three repetitions, flash at 32,768 tokens without the trigger; K3/32B at their stated budgets); \textbf{C3} = budget-matched greedy runs (temperature 0, 8,192-token budget, vLLM; three repetitions for the Qwen3-4B agent and the per-component ablation groups, single repetition for PT-4B, the format-only baselines, and the $\kappa$ ablation). \ref{sec:budget-notes} records two engineering notes on generation budgets and API-side behavior across the campaigns, including a silent endpoint upgrade that we detected and excluded from pooling.

\paragraph{PT-4B training} PT-4B is a TCM post-trained variant of Qwen3-4B. Starting from the unchanged Qwen3-4B base checkpoint, we applied supervised fine-tuning (SFT) on a self-curated TCM corpus---one subset of which carries chain-of-thought rationales---followed by group relative policy optimization (GRPO) reinforcement post-training, following the recipes of DeepSeekMath~\cite{deepseekmath} and DeepSeek-R1~\cite{deepseekr1}. A first version (V1) was trained in April 2025; all experiments in this paper use the V2 upgrade checkpoint of 2025-10-09. Post-training data are drawn from our private TCM datasets and are disjoint from the TCMEval-SDT benchmark (Section~\ref{sec:hygiene}).

\section{Results}
\label{sec:results}

\subsection{The boundary condition: protocol gains are sign-conditional on zero-shot output discipline}
\label{sec:significance}

The cleanest evidence for the mechanism's contribution is the same-model comparison: a base model zero-shot versus the same model with our agent, where no component of the base model changes---only the reasoning protocol. Table~\ref{tab:main} reports official-script scores and hit rates across five base models; Figure~\ref{fig:results} visualizes the comparison; Table~\ref{tab:discipline} profiles each model's zero-shot output discipline; Table~\ref{tab:sig} reports the paired significance tests.

\begin{figure*}[!t]
\centering
\includegraphics[width=\textwidth]{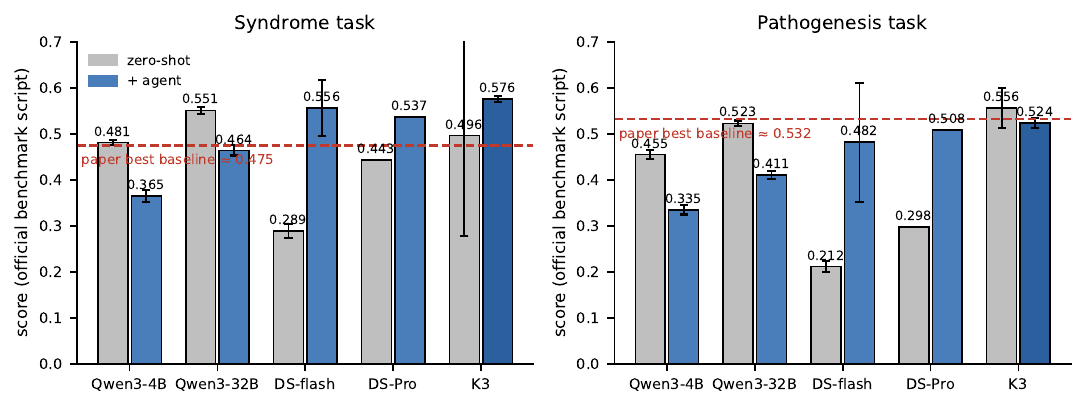}
\caption{Official-script scores on TCMEval-SDT (our own runs, n=200 per task). Paired bars show zero-shot vs.\ +agent for each base model; error bars mark the standard deviation over three repetitions where available. Dashed lines mark the benchmark paper's best reported zero-shot baselines (2023--2024 era, 50-case subset). Note the K3 pair: the zero-shot error bar ($\pm0.218$, descriptive spread over the $n{=}3$ runs) versus the agent's ($\pm0.007$) is the visual evidence for the stabilization claim---the protocol reduces run-to-run standard deviation by an order of magnitude (point estimate 31-fold).}
\label{fig:results}
\end{figure*}

\begin{table}[t]
\centering
\caption{Official-script scores on TCMEval-SDT (our own runs, n=200 per task). Values are mean $\pm$ std over 3 repetitions; $\dagger$ marks single runs; $^{\ddagger}$ marks single budget-matched greedy runs (8,192-token budget). Paper baselines are 2023--2024-era zero-shot on a 50-case subset, converted to per-case means. Bold values mark per-column maxima; the bold row label marks the protocol configuration.}
\label{tab:main}
\small
\resizebox{\columnwidth}{!}{%
\begin{tabular}{@{}llcc@{}}
\toprule
Model & Variant & Syndrome & Pathogen. \\
\midrule
Qwen3-4B & zero-shot & 0.481 $\pm$ 0.005 & 0.455 $\pm$ 0.010 \\
Qwen3-4B & + agent & 0.365 $\pm$ 0.013 & 0.335 $\pm$ 0.010 \\
PT-4B & zero-shot & 0.473$^{\ddagger}$ & 0.457$^{\ddagger}$ \\
PT-4B & + agent & 0.362$^{\ddagger}$ & 0.345$^{\ddagger}$ \\
\midrule
Qwen3-32B & zero-shot & 0.551 $\pm$ 0.007 & 0.523 $\pm$ 0.006 \\
Qwen3-32B & + agent & 0.464 $\pm$ 0.011 & 0.411 $\pm$ 0.009 \\
\midrule
DeepSeek-V4-Flash & zero-shot & 0.289 $\pm$ 0.015 & 0.212 $\pm$ 0.012 \\
DeepSeek-V4-Flash & + agent & 0.556 $\pm$ 0.061 & 0.482 $\pm$ 0.129 \\
\midrule
DeepSeek-V4-Pro & zero-shot & 0.443$^\dagger$ & 0.298$^\dagger$ \\
DeepSeek-V4-Pro & + agent & 0.537$^\dagger$ & 0.508$^\dagger$ \\
\midrule
Kimi K3 & zero-shot & 0.496 $\pm$ 0.218 & \textbf{0.556} $\pm$ 0.043 \\
Kimi K3 & \textbf{+ agent} & \textbf{0.576} $\pm$ 0.007 & 0.524 $\pm$ 0.011 \\
\midrule
\emph{Paper best baseline} & & \emph{$\approx$0.475} & \emph{$\approx$0.532} \\
\bottomrule
\end{tabular}%
}
\end{table}

\begin{table}[t]
\centering
\caption{Auxiliary metrics (n=200 per task, C2/C3 campaigns as in Table~\ref{tab:main}): exact match (EM, \%), macro-F1, and set hit rate (\%). The hit--score gap quantifies the over-prediction failure mode of Section~\ref{sec:failures}: high hit with a modest official-script score indicates over-selection (e.g., 4B zero-shot: 87.0\% hit vs.\ 0.481 score).}
\label{tab:aux}
\small
\resizebox{\columnwidth}{!}{%
\begin{tabular}{@{}llcccccc@{}}
\toprule
& & \multicolumn{3}{c}{Syndrome} & \multicolumn{3}{c}{Pathogen.} \\
\cmidrule(lr){3-5}\cmidrule(lr){6-8}
Model & Variant & EM & F1 & hit & EM & F1 & hit \\
\midrule
Qwen3-4B & zero & 20.0 & 0.588 & 87.0 & 9.5 & 0.578 & 92.5 \\
Qwen3-4B & agent & 16.5 & 0.465 & 70.0 & 7.0 & 0.425 & 72.5 \\
PT-4B & zero & 17.5 & 0.589 & 89.0 & 8.0 & 0.585 & 92.5 \\
PT-4B & agent & 15.0 & 0.444 & 67.0 & 8.5 & 0.452 & 77.0 \\
Qwen3-32B & zero & 29.7 & 0.644 & 88.0 & 16.7 & 0.636 & 92.0 \\
Qwen3-32B & agent & 20.2 & 0.563 & 83.0 & 11.3 & 0.523 & 83.3 \\
DS-Flash & zero & 22.8 & 0.310 & 36.0 & 10.7 & 0.244 & 31.8 \\
DS-Flash & agent & 30.9 & 0.661 & 88.9 & 17.3 & 0.615 & 87.1 \\
DS-Pro & zero & 32.0 & 0.484 & 58.0 & 12.0 & 0.352 & 47.5 \\
DS-Pro & agent & 25.0 & 0.636 & 89.0 & 12.5 & 0.629 & 93.0 \\
K3 & zero & 28.5 & 0.576 & 80.8 & 23.3 & 0.657 & 90.5 \\
K3 & agent & 33.5 & 0.662 & 87.8 & 17.0 & 0.639 & 92.0 \\
\bottomrule
\end{tabular}%
}
\end{table}

\begin{table}[t]
\centering
\caption{Output-discipline profile of each base model, measured on protocol-free zero-shot outputs (200 cases per task), against the protocol's gain $\Delta$ (agent $-$ zero-shot; CIs and $p$-values in Table~\ref{tab:sig}). \textbf{Answer-extraction rate}: fraction of cases whose output yields a machine-extractable answer under the benchmark's official extraction rule, computed on the first repetition of each arm (DS-Pro's only repetition), so the rate is comparable across arms---for thinking models this coincides with termination within budget, and it subsumes format non-compliance. \textbf{Over-prediction ratio}: per-case mean of $|\mathrm{pred}|/|\mathrm{gold}|$. \textbf{Run std}: standard deviation of run means across three repetitions. \textbf{Cross-run agreement}: fraction of cases with identical predicted label sets across all three runs. Run std and self-consistency require three repetitions and are unavailable for the single-run zero-shot configuration (DS-Pro).}
\label{tab:discipline}
\small
\resizebox{\columnwidth}{!}{%
\begin{tabular}{@{}llccccc@{}}
\toprule
Model & Task & Answer rate (\%) & Over-pred. & Run std & Cross-run agr.\ (\%) & $\Delta$ \\
\midrule
DS-Flash & syn. & 38.5 & 0.42 & 0.015 & 60.0 & $+0.267$ \\
DS-Flash & pat. & 33.0 & 0.42 & 0.012 & 62.5 & $+0.271$ \\
DS-Pro & syn. & 62.5 & 0.75 & --- & --- & $+0.094$ \\
DS-Pro & pat. & 51.0 & 0.66 & --- & --- & $+0.210$ \\
K3 & syn. & 99.5 & 1.38 & 0.218 & 10.5 & $+0.080$ \\
K3 & pat. & 97.0 & 1.43 & 0.043 & 37.0 & $-0.032$ \\
Qwen3-32B & syn. & 100 & 1.29 & 0.007 & 53.5 & $-0.087$ \\
Qwen3-32B & pat. & 100 & 1.28 & 0.006 & 38.5 & $-0.112$ \\
Qwen3-4B & syn. & 99.5 & 1.46 & 0.005 & 82.0 & $-0.116$ \\
Qwen3-4B & pat. & 100 & 1.42 & 0.010 & 69.0 & $-0.120$ \\
\bottomrule
\end{tabular}%
}
\end{table}

\begin{figure*}[t]
\centering
\includegraphics[width=\textwidth]{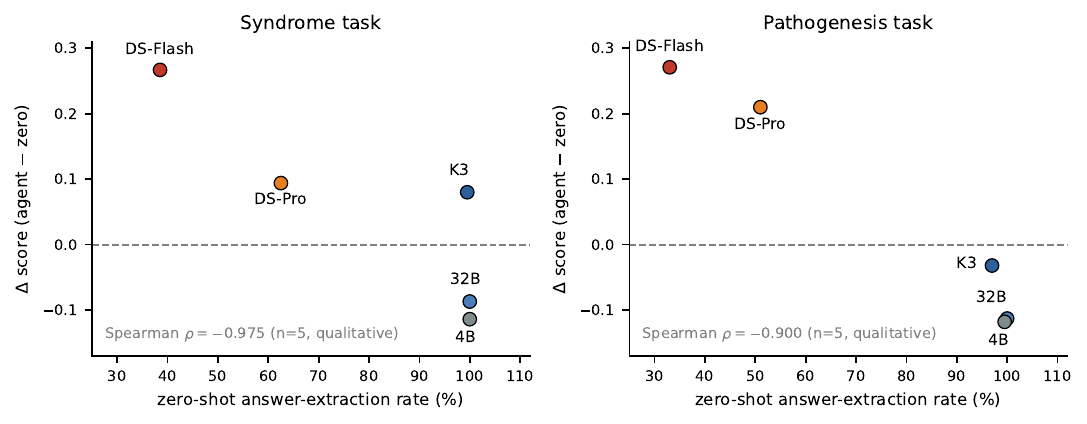}
\caption{The protocol's gain against the zero-shot answer-extraction rate, per base model and task. Each model the protocol helps carries a distinct measurable failure mode (non-termination for flash, partial non-termination for Pro, instability for K3); the two models it hurts show none. Rank correlations are reported qualitatively at $n{=}5$.}
\label{fig:discipline}
\end{figure*}

\begin{table}[t]
\centering
\caption{Paired tests of the mechanism's gain (zero-shot $\rightarrow$ +agent) per base model and task: per-case mean difference $\Delta$ with bootstrap 95\% CI and Wilcoxon $p$ (n=200 cases). Raw $p$-values; the family-wise Bonferroni threshold for this 10-test family is 0.005. The Qwen3-4B comparison is budget-matched (both sides greedy with an 8,192-token budget, three repetitions each on the same vLLM backend). DS-Pro rows are single-repetition (case-level tests only) and are descriptive under the two-level standard of Section~\ref{sec:stats}. The 4B negative result cannot be attributed to truncation. \textbf{Run-level $p$}: conservative replication check---Welch's $t$-test over the three run means per side (Section~\ref{sec:stats}), uncorrected; ``---'' for configurations with a single-run side.}
\label{tab:sig}
\small
\resizebox{\columnwidth}{!}{%
\begin{tabular}{@{}llcccc@{}}
\toprule
Model & Task & $\Delta$ & 95\% CI & $p$ (case) & $p$ (run) \\
\midrule
Qwen3-4B & syndrome & $-$0.116 & [$-$0.167, $-$0.064] & 2.6e-05 & 1.7e-03 \\
Qwen3-4B & pathogen. & $-$0.120 & [$-$0.157, $-$0.084] & 8.1e-10 & 1.0e-04 \\
DS-Flash & syndrome & +0.267 & [+0.226, +0.307] & 6.1e-22 & 1.3e-02 \\
DS-Flash & pathogen. & +0.271 & [+0.234, +0.307] & 2.0e-24 & 6.7e-02 \\
DS-Pro & syndrome & +0.094 & [+0.039, +0.149] & 4.2e-03 & --- \\
DS-Pro & pathogen. & +0.210 & [+0.131, +0.295] & 1.4e-12 & --- \\
K3 & syndrome & +0.080 & [+0.047, +0.113] & 3.5e-06 & 0.59 \\
K3 & pathogen. & $-$0.032 & [$-$0.061, $-$0.004] & 3.1e-02 & 0.32 \\
Qwen3-32B & syndrome & $-$0.087 & [$-$0.126, $-$0.048] & 5.0e-05 & 9e-04 \\
Qwen3-32B & pathogen. & $-$0.112 & [$-$0.147, $-$0.079] & 2.9e-09 & 1e-04 \\
\bottomrule
\end{tabular}%
}
\end{table}

\begin{table}[t]
\centering
\caption{Rescue/terminating decomposition of the zero-shot$\rightarrow$agent gain on the DeepSeek models (paired per-case differences, 95\% CI, Wilcoxon $p$). \textbf{Reliably terminating}: zero-shot yields an extractable answer in \emph{every} repetition (DS-Pro: single run)---a strictly stronger condition than the per-repetition answer-extraction rate of Table~\ref{tab:discipline}, which is why the flash subset (43/200 cases) is smaller than that rate (38.5\%) suggests. \textbf{Rescue-prone}: the complement. K3 and the Qwen models terminate at $\geq$97\% zero-shot answer rate, so for them the full-set contrast in Table~\ref{tab:sig} is already a terminating-case estimate. For DS-Pro the agent-side budget is 16,384 tokens and roughly half of its terminating case-runs exceed the 8k-equivalent line, so its rescue-prone gains retain a budget component that the post-hoc proxy (Appendix~\ref{sec:budget-neutralization}) can only partially exclude. DS-Pro rows are single-repetition and hence descriptive under the two-level standard of Section~\ref{sec:stats}. Full-set rows duplicate Table~\ref{tab:sig} for readability.}
\label{tab:decomp}
\small
\resizebox{\columnwidth}{!}{%
\begin{tabular}{@{}llccccc@{}}
\toprule
Model & Subset & $n$ & Zero & Agent & $\Delta$ (95\% CI) & $p$ \\
\midrule
DS-Flash syn. & full & 200 & 0.289 & 0.556 & $+0.267$ [$+0.226, +0.307$] & 6.1e-22 \\
 & reliably terminating & 43 & 0.866 & 0.781 & $-0.085$ [$-0.133, -0.037$] & 2.0e-03 \\
 & rescue-prone & 157 & 0.132 & 0.494 & $+0.363$ [$+0.325, +0.401$] & 2e-25 \\
\midrule
DS-Flash pat. & full & 200 & 0.212 & 0.482 & $+0.271$ [$+0.233, +0.308$] & 2e-24 \\
 & reliably terminating & 32 & 0.709 & 0.622 & $-0.087$ [$-0.140, -0.034$] & 3.1e-03 \\
 & rescue-prone & 168 & 0.117 & 0.456 & $+0.339$ [$+0.304, +0.374$] & 4e-27 \\
\midrule
DS-Pro syn. & full & 200 & 0.443 & 0.537 & $+0.094$ [$+0.039, +0.149$] & 4.2e-03 \\
 & reliably terminating & 125 & 0.708 & 0.598 & $-0.111$ [$-0.164, -0.058$] & 2.1e-04 \\
 & rescue-prone & 75 & 0.000 & 0.435 & $+0.435$ [$+0.368, +0.502$] & 6e-12 \\
\midrule
DS-Pro pat. & full & 200 & 0.298 & 0.508 & $+0.210$ [$+0.159, +0.260$] & 1e-12 \\
 & reliably terminating & 102 & 0.584 & 0.534 & $-0.050$ [$-0.097, -0.002$] & 6.0e-02 \\
 & rescue-prone & 98 & 0.000 & 0.480 & $+0.480$ [$+0.429, +0.531$] & 1e-16 \\
\bottomrule
\end{tabular}%
}
\end{table}

Four controlled within-model contrasts, each holding model and configuration fixed, establish the boundary condition.

\textbf{Rescue where discipline fails.} On DeepSeek-V4-Flash (DS-Flash)---whose zero-shot output fails to yield an answer on 61.5\%/67.0\% of cases (Table~\ref{tab:discipline})---the budget-corrected three-repetition campaign yields the largest significant gain: $+0.267$ on syndrome (95\% CI $[+0.226, +0.307]$, $p=6.1{\times}10^{-22}$) and $+0.271$ on pathogenesis ($p=2.0{\times}10^{-24}$). On DeepSeek-V4-Pro (DS-Pro)---single repetition on both sides, hence descriptive under the two-level standard of Section~\ref{sec:stats}---the gain is positive on both tasks ($+0.094$, case-level $p=4.2{\times}10^{-3}$; $+0.210$, $p=1.4{\times}10^{-12}$), with only 1.5\% empty outputs on the agent side; being run-level unreplicated, we report it as descriptive rather than headline evidence. Table~\ref{tab:decomp} decomposes these gains. For each DeepSeek model we split the 200 cases into a \textbf{reliably terminating subset}---cases whose zero-shot output yields an extractable answer in every repetition---and the complementary \textbf{rescue-prone subset}. On rescue-prone cases the protocol's effect is large and positive (flash, three repetitions: $+0.363$/$+0.339$; DS-Pro, single repetition, descriptive: $+0.435$/$+0.480$): it converts non-answers into answers. On reliably terminating cases the effect flips sign and is small: $-0.085/-0.087$ (flash, $p{=}0.002/0.003$) and $-0.111/-0.050$ (Pro, $p{=}2{\times}10^{-4}$/n.s.; single repetition, descriptive). The full-set gains on DeepSeek models are therefore \emph{entirely} termination rescue; no positive remainder survives on the truncation-free subset.

\paragraph{Post-hoc budget neutralization via elapsed-time certification} One objection to the rescue attribution remains: the flash agent runs at 32,768 tokens against the zero-shot baseline's 8,192, so part of the rescue could reflect budget rather than discipline. Per-case token counts were not archived, but wall-clock times were, and they afford a \emph{rate-free} certification: the fastest zero-shot case-run that exhausted the 8,192-token budget took 66.9\,s (syndrome) / 64.6\,s (pathogenesis), so any agent case-run faster than that bound necessarily consumed fewer than 8,192 tokens. 72\%/75\% of terminating agent case-runs pass the bound (median 49\,s/51\,s). On rescue-prone cases certified under the strictest reading---the agent terminates below the bound in \emph{every} repetition---the gain persists or grows: $+0.464$ on syndrome (95\% CI $[+0.332,+0.597]$, $p{=}3.2{\times}10^{-4}$, $n{=}17$) and $+0.237$ on pathogenesis (CI $[+0.140,+0.335]$, $p{=}0.036$, $n{=}7$); that the syndrome gain \emph{increases} when high-consumption cases are excluded is the opposite of what a budget explanation predicts. A well-powered typical-case certification (median repetition below the bound; $n{=}136/87$) gives $+0.369/+0.399$ ($p{=}3.6{\times}10^{-22}$/$2.1{\times}10^{-15}$). We do not claim full neutralization: about a quarter of terminating agent case-runs exceed the 8k-equivalent line, and for DS-Pro (16,384-token agent budget) roughly half do, leaving certified subsets too small to be informative ($n{=}2/12$)---Pro's rescue-prone gains therefore retain a budget component this proxy can only partially exclude. Calibration, subset definitions, sensitivity across certification variants, and the non-termination pathology evidence are in Appendix~\ref{sec:budget-neutralization}.

On flash the protocol also raises the set-hit rate from 32.5\%/29.5\% to 66.0\%/64.5\% (syndrome/pathogenesis; \textbf{C1}, the August-2026 initial run---the corresponding C2 figures in Table~\ref{tab:aux} differ because the API-side model drifted between campaigns, \ref{sec:budget-notes}).

\textbf{Stabilization of a high-variance model.} K3's forced temperature 1.0 makes its zero-shot output unstable (run means span 0.245--0.636 on syndrome). Under the protocol the run-to-run standard deviation drops from 0.218 to 0.007---a point-estimate factor of 31 (variance-ratio $F$-test $p{=}0.002$, 95\% CI [5, 200]; pathogenesis: 3.9-fold, n.s., $p{=}0.12$, CI [0.6, 25]). The wide interval reflects three runs per side; we report the effect as an order-of-magnitude reduction, not a precise factor. On score, K3 syndrome gains $+0.080$ at the case level ($p{=}3.5{\times}10^{-6}$) but the conservative run-level check is non-significant ($p{=}0.59$, $n{=}3$ runs); pathogenesis shows a small nominal decrease ($-0.032$, n.s.\ after Bonferroni). We therefore claim stabilization as the robust K3 finding, and the score gain as case-level evidence only.

\textbf{Significant cost where discipline is internalized.} On Qwen3-32B---whose zero-shot output terminates at 100\% and scores 0.551/0.523, near the benchmark's best reported baselines---the agent scores \emph{significantly below} zero-shot (syndrome $-0.087$, $p=5.0{\times}10^{-5}$; pathogenesis $-0.112$, $p=2.9{\times}10^{-9}$). A temperature control (agent at temperature 0 vs.\ 0.6) rules out decoding as the cause: syndrome is unchanged ($p=0.31$), and pathogenesis at temperature 0 degrades further (0.201 vs.\ 0.411, $p=4{\times}10^{-14}$). On Qwen3-4B the budget-matched greedy comparison (both sides at 8,192 tokens) replicates the cost ($-0.116$, $p=2.6{\times}10^{-5}$; $-0.120$, $p=8.1{\times}10^{-10}$), so the negative result cannot be attributed to truncation. The TCM post-trained variant PT-4B shows no significant post-training effect on either side ($|\Delta|<0.03$, all n.s.), and the protocol's cost replicates on it as well ($-0.111/-0.112$, $p{<}10^{-4}$)---base-model behavior, not domain adaptation, dominates. PT-4B is moreover a 4B with reasoning chains internalized (its SFT corpus includes chain-of-thought rationales); that even a CoT-trained model is changed so little by the protocol suggests the protocol's leverage point is not the reasoning chain itself.

\textbf{What is controlled and what is not.} The cross-model pattern---gains ordered against zero-shot discipline in Table~\ref{tab:discipline} and Figure~\ref{fig:discipline}---is qualitative only: the points were measured under per-model best-effort configurations (different generation budgets, decoding temperatures, and campaign dates, including an API-side behavior drift for flash), so no controlled dose--response curve can be drawn across them. What \emph{is} controlled evidence are the four within-model contrasts above. Read together, they support the narrower claim that \textbf{the mechanism helps where zero-shot output discipline is poor (non-termination, format failure, high variance) and costs where discipline is already strong}; the dose--response quantification across a unified protocol is future work. A single run can be badly misleading in either direction: K3's first run suggested a +158\% pathogenesis gain, which the three-repetition campaign replaced with a small nominal decrease---single-run comparisons are anecdotes; repeated runs with paired tests are evidence.

\subsection{Mechanistic attribution: format--content complementarity}
\label{sec:ablation}

Which layer carries these effects? The G1 ablation (Table~\ref{tab:ablation}) removes the matrix requirement and the output-format mandate jointly (both live in a single instruction block in the deployed prompt), so it cannot separate them. The \emph{format-only} baseline---candidate markers and the mandatory single-line output format, but no matrix and no knowledge injection---closes this gap on both the budget-matched 4B and DeepSeek-V4-Flash. Together with G1 (content retained; matrix+format block removed), the 4B cells form a two-factor pattern (syndrome/pathogenesis): zero-shot 0.481/0.455; format-only 0.240/0.172 ($-0.241/-0.283$, $p{<}10^{-12}$); G1 0.388/0.319 ($-0.093/-0.136$ vs.\ zero-shot, $p{=}0.004$/$2.1{\times}10^{-8}$); full protocol 0.365/0.335. The structure is \emph{not} additive. The format layer costs or is neutral under both content conditions ($-0.241/-0.283$ without content; $-0.022/+0.015$ with it, G1$\rightarrow$full, n.s.), and the content bundle (falsification matrix + knowledge injection + causal chain) is significantly harmful without the format ($-0.093/-0.136$) yet strongly positive with it ($+0.126/+0.162$, $p{=}1.4{\times}10^{-5}$/$3.7{\times}10^{-10}$)---a sign-flipping interaction. On flash the chain reads zero-shot 0.289/0.212 $\rightarrow$ format-only 0.413/0.204 ($+0.124/-0.007$; $p{=}10^{-5}$/n.s.) $\rightarrow$ full 0.556/0.482 ($+0.143/+0.278$, $p{<}10^{-7}$).

Three conclusions follow, each narrower than an additive two-layer story. (i) \textbf{The format/candidate-bounding layer is sign-conditional}: it rescues the truncation-prone model (flash syndrome $+0.124$; answer-extraction rate 38.5\%$\rightarrow$75.0\% syndrome, 33.0\%$\rightarrow$49.5\% pathogenesis) but costs the disciplined one (4B), suppressing free causal reasoning into candidate oscillation; the 4B negative headline of Section~\ref{sec:significance} is attributable primarily to this layer. (ii) \textbf{The causal chain and the format layer are complements, not additive layers}: on 4B the causal chain pays off only jointly with the format, and the sign flips across the two content conditions; the matrix and knowledge marginals are separately nil (see (iii)). (iii) \textbf{The matrix's own marginal contribution to score is nil-to-negative on this model.} Two independent estimates agree: with format present, full $-$ (content+format, no matrix) $= -0.027/-0.008$ ($p{=}0.12/0.57$, n.s.); without format, (content+matrix, no format) $-$ (content only) $= -0.054/-0.036$ ($p{=}0.018/0.17$). The bundle's positive residual therefore comes from the causal chain conditional on format---not from the matrix requirement and not from knowledge injection: the residual with the matrix excluded is $+0.153/+0.171$ over format-only ($p{<}10^{-5}$), and removing \emph{all} injected knowledge from the full protocol leaves scores unchanged ($-0.001/+0.019$, n.s.; G4, Table~\ref{tab:ablation}). The ordered attribution on this model thus reads: format-only $\rightarrow$ $+$causal chain (the main effect) $\rightarrow$ $+$matrix ($\approx$0) $\rightarrow$ $+$knowledge ($\approx$0). The format layer's own marginal, with content and matrix present, is $+0.031/+0.051$ (n.s./$p{=}8{\times}10^{-4}$). Where the matrix-and-format block \emph{does} act measurably is the candidate space itself (Table~\ref{tab:factorial}): free-form zero-shot hit rate is 87.5\%, dropping to 59.0\% with content alone, and recovering to 67.5--74.2\% as the matrix and the emission format are added---the protocol's demonstrated functions are candidate-space enforcement, the cell-level audit trail, and the misinformation detection of Section~\ref{sec:misinfoexp}, not benchmark score on this model.

\begin{table}[t]
\centering
\caption{The completed three-factor factorial on budget-matched Qwen3-4B (three repetitions per arm, n=200 per task). The content bundle is knowledge injection + causal chain; ``format'' is the candidate-bounded single-line emission mandate; ``matrix'' is the falsification-matrix requirement.}
\label{tab:factorial}
\small
\resizebox{\columnwidth}{!}{%
\begin{tabular}{@{}lccc@{}}
\toprule
Arm (content / matrix / format) & Syn. & Pat. & Hit syn./pat.\ (\%) \\
\midrule
zero-shot ($-$/$-$/$-$) & 0.481 $\pm$ 0.005 & 0.455 $\pm$ 0.010 & 87.5 / 92.0 \\
format only ($-$/$-$/+) & 0.240 & 0.172 & 38.0 / 38.0 \\
content only (G1, +/$-$/$-$) & 0.388 $\pm$ 0.009 & 0.319 $\pm$ 0.007 & 59.0 / 66.7 \\
no matrix (G5, +/$-$/+) & 0.392 $\pm$ 0.007 & 0.343 $\pm$ 0.004 & 67.3 / 72.8 \\
no format (G6, +/+/$-$) & 0.334 $\pm$ 0.007 & 0.284 $\pm$ 0.002 & 67.5 / 70.5 \\
full protocol (+/+/+) & 0.365 $\pm$ 0.013 & 0.335 $\pm$ 0.010 & 67.5 / 74.2 \\
\bottomrule
\end{tabular}%
}
\end{table}

\begin{table}[t]
\centering
\caption{Component ablations (Qwen3-4B, C3 budget-matched greedy configuration, 8,192-token budget, vLLM; three repetitions per group, n=200 per task). Official score mean $\pm$ std; $\Delta$ and hit-rate change vs.\ G0 (paired Wilcoxon; n.s. = $p>0.05$).}
\label{tab:ablation}
\small
\resizebox{\columnwidth}{!}{%
\begin{tabular}{@{}lcccccc@{}}
\toprule
& \multicolumn{2}{c}{Official-script score} & \multicolumn{2}{c}{$\Delta$ vs.\ G0 (CI, $p$)} & \multicolumn{2}{c}{Hit $\Delta$ (pp)} \\
\cmidrule(lr){2-3}\cmidrule(lr){4-5}\cmidrule(lr){6-7}
Group & Syn. & Pat. & Syn. & Pat. & Syn. & Pat. \\
\midrule
G0 full agent & 0.365 $\pm$ 0.013 & 0.335 $\pm$ 0.010 & --- & --- & 67.5 & 74.2 \\
G1 $-$ matrix+format block & 0.388 $\pm$ 0.009 & 0.319 $\pm$ 0.007 & +0.022 (n.s.) & $-0.015$ (n.s.) & \textbf{$-8.5$} & \textbf{$-7.5$} \\
G2 $-$ evidence checklist & 0.403 $\pm$ 0.009 & 0.332 $\pm$ 0.004 & +0.038 (n.s.) & $-0.003$ (n.s.) & +6.2 & $-1.6$ \\
G3 $-$ diff.\ guide & 0.391 $\pm$ 0.018 & 0.347 $\pm$ 0.015 & +0.026 (n.s.) & +0.012 (n.s.) & +2.2 & $-0.7$ \\
G4 $-$ all skill knowledge & 0.364 $\pm$ 0.009 & 0.353 $\pm$ 0.004 & $-0.001$ (n.s.) & +0.019 (n.s.) & \textbf{+11.9} & \textbf{+4.8} \\
\bottomrule
\end{tabular}%
}
\end{table}

\subsection{Standard reasoning baselines: the protocol is not a CoT variant}
\label{sec:stdbaselines}

Could the protocol be merely an expensive chain-of-thought variant? Two standard reasoning baselines answer this on budget-matched Qwen3-4B (Table~\ref{tab:stdbaselines}). Plain CoT prompting (greedy, three repetitions) scores 0.487/0.471---statistically indistinguishable from zero-shot and, consistently with the boundary condition, \emph{above} the agent; a three-sample self-consistency vote at temperature 0.7 scores 0.479/0.405, recovering the sampling penalty without exceeding greedy decoding. Structurally, neither baseline shares the protocol's mechanism: CoT and self-consistency impose no candidate space (their hit rates, 88.7\%/87.5\%, are the free-form regime's, against the protocol's 67.5--74.2\%), no machine-checkable verification checkpoint, and no audit trail---they answer a different question (``what conclusion sounds right'') from the protocol's (``which candidates survive falsification against named evidence''). The score comparison is therefore not the protocol's claim to value; that claim is governance (Section~\ref{sec:floorceiling}).

Two honest datapoints delimit the claim. On K3---the high-variance model---a three-sample self-consistency vote (its runs are already at the forced temperature 1.0) scores \textbf{0.631/0.596}, \emph{above} both K3 zero-shot (0.496/0.556) and K3+agent (0.576/0.524), at 3$\times$ inference cost and zero auditability; and on 4B, both CoT and self-consistency outscore the agent. Where raw multi-select score is the only objective and auditability is not required, majority voting over sampled runs is the better tool for an unstable model; the protocol's contribution lies elsewhere (deterministic termination, candidate-space enforcement, cell-level auditability, safety checks, and misinformation detection). Flash-side CoT and self-consistency baselines are deferred: the API endpoint was silently redirected to a newer model generation mid-campaign (\ref{sec:budget-notes}), so fresh flash measurements would not be poolable with the reported C2 campaign.

\begin{table}[t]
\centering
\caption{Standard reasoning baselines vs.\ the agent (same-model, official-script score; three repetitions). K3 self-consistency uses the existing temperature-1.0 repetitions; 4B self-consistency samples at temperature 0.7.}
\label{tab:stdbaselines}
\small
\resizebox{\columnwidth}{!}{%
\begin{tabular}{@{}llcc@{}}
\toprule
Model & Configuration & Syn. & Pat. \\
\midrule
Qwen3-4B & zero-shot & 0.481 $\pm$ 0.005 & 0.455 $\pm$ 0.010 \\
Qwen3-4B & CoT (greedy) & 0.487 $\pm$ 0.004 & 0.471 $\pm$ 0.005 \\
Qwen3-4B & self-consistency ($t{=}0.7$) & 0.479 & 0.405 \\
Qwen3-4B & \textbf{+ agent (full protocol)} & 0.365 $\pm$ 0.013 & 0.335 $\pm$ 0.010 \\
\midrule
K3 & zero-shot & 0.496 $\pm$ 0.218 & 0.556 $\pm$ 0.043 \\
K3 & self-consistency ($t{=}1.0$) & \textbf{0.631} & \textbf{0.596} \\
K3 & \textbf{+ agent (full protocol)} & 0.576 $\pm$ 0.007 & 0.524 $\pm$ 0.011 \\
\bottomrule
\end{tabular}%
}
\end{table}

\subsection{Protocol--knowledge decoupling and specialty routing}
\label{sec:contamination}

The knowledge-side ablations (G2--G4, Table~\ref{tab:ablation}) yield the paper's most counterintuitive---and, for the framework's design, most informative---result. Removing \emph{all} skill-knowledge injection (G4) leaves official-script scores unchanged ($-0.001$/$+0.019$, n.s.) while \emph{raising} the hit rate by 11.9/4.8 pp. In other words, \textbf{the protocol's measured effect on this benchmark is carried by its structure, not by the injected knowledge content}. We read this as evidence of protocol--knowledge decoupling: the reasoning discipline and the knowledge layer vary independently, which is precisely the property the framework's architecture is designed to exploit---knowledge remains external, expert-revisable data that can be corrected or swapped without touching either model weights or control flow (Section~\ref{sec:knowledge}). A framework whose protocol effect depended on knowledge content would be coupled; ours is not.

The hit-rate increase under G4 has a concrete mechanism. Removing the matrix-and-format block (G1) costs 8.5/7.5 pp of hit rate---conclusions escape the candidate space, an effect attributable to the lost candidate-bounded emission mandate---whereas removing knowledge (G4) \emph{raises} hit rate: injected knowledge, as currently organized, pulls conclusions \emph{out} of the candidate space without adding score. The failure log locates the defect. The differential guide was tuned on 539 insomnia cases and was misplaced in the always-loaded base skill; on broad-domain benchmark cases it biased answers toward insomnia syndromes (e.g., a pediatric nephritis case answered ``heart--kidney disharmony,'' an insomnia syndrome absent from the candidate list) and injected misleading decision anchors. Consistently, removing the \emph{evidence checklist} (G2) or the \emph{insomnia-tuned differential guide} (G3) \emph{tends to improve} scores---directionally consistent across tasks and repetitions but below significance ($+0.038/+0.026$ and $+0.026/+0.012$, n.s.).

Together these results support, at modest statistical strength, the framework's design principle that \textbf{specialty knowledge must be routed by specialty rather than universally injected}, and they identify a concrete iteration---relocate specialty-tuned knowledge files into their specialty skills---of exactly the kind the white-box audit loop of Section~\ref{sec:audit} is designed to enable. They also delimit honestly what the knowledge architecture does \emph{not} currently do: contribute measurable benchmark score. We report this as a boundary on the architecture's present value, not as a defect to be hidden.

\subsection{Stopping-fraction calibration}
\label{sec:threshold}

The sufficiency rule's stopping fraction $\kappa$ (at least a $\kappa$-fraction of candidates must be falsified) was varied over $\{\tfrac14,\tfrac12,\tfrac34\}$ on the budget-matched Qwen3-4B configuration (vLLM, greedy, 8,192-token budget; 200 cases per task, one repetition per $\kappa$; only the sufficiency wording changed---Table~\ref{tab:kappa}). $\kappa{=}\tfrac34$ wins on both tasks (pathogenesis $+0.037$ vs.\ the deployed $\tfrac12$, 95\% CI $[+0.001,+0.075]$, $p=0.036$; syndrome $+0.024$, n.s.), and $\kappa{=}\tfrac14$ loses on syndrome ($-0.012$, n.s.). Two conclusions, both without overreach: (i) the deployed default is \emph{not} optimal on this benchmark, and the profitable direction is \emph{stricter} falsification---consistent with the over-prediction failure mode of Section~\ref{sec:failures}, which predicts exactly this direction; (ii) the stopping fraction is a meaningful, tunable operating point rather than a constant of nature, warranting per-deployment calibration rather than a universal $\kappa$.

\begin{table}[t]
\centering
\caption{Stopping-fraction ablation (Qwen3-4B, budget-matched greedy, 8,192-token budget, vLLM; n=200 per task, one repetition per $\kappa$). Official scores; $\Delta$ and $p$ vs.\ the deployed $\kappa{=}\tfrac12$ (paired Wilcoxon with bootstrap CI). The $\kappa{=}\tfrac12$ row is a single run of the G0 configuration and sits within G0's three-repetition run-to-run spread ($0.365 \pm 0.013$ syndrome, Table~\ref{tab:ablation}); the within-table $\Delta$ contrasts are unaffected since all three $\kappa$ arms share the same single-repetition protocol.}
\label{tab:kappa}
\small
\begin{tabular}{@{}lccccc@{}}
\toprule
$\kappa$ & Syndrome & Pathogen. & $\Delta$ syn. & $\Delta$ pat. & $p$ (pat.) \\
\midrule
$\tfrac14$ & 0.368 & 0.333 & $-0.012$ & $+0.010$ & n.s. \\
$\tfrac12$ (deployed) & 0.380 & 0.323 & --- & --- & --- \\
$\tfrac34$ & \textbf{0.403} & \textbf{0.361} & $+0.024$ & $+0.037$ & 0.036 \\
\bottomrule
\end{tabular}
\end{table}

\subsection{Cross-domain transfer check: premise falsification under injected misinformation}
\label{sec:misinfoexp}

Sections~\ref{sec:contamination-threat} and \ref{sec:discussion} argue that the matrix protocol should transfer to contaminated-reasoning settings. The protocol tested here was neither derived from nor inspired by the misinformation literature: it was developed independently for clinical syndrome differentiation, and we use the benchmark of~\cite{misinfoprop} purely as a public external testbed---their data and their instruction intervention, our protocol. We test it on their math word problems with perturbed premise equations (a \emph{preemptive answer attack} (PAA) scenario in the reasoning-safety taxonomy of~\cite{reasoningsafety}). Their original LFS-hosted test data are no longer available, so we reconstructed 243 unique questions (question, gold answer, correct premises, perturbed premises) from the repository's retained evaluation aggregate, and evaluate DeepSeek-V4-Flash under four conditions (3 samples per question, their verbatim prompts): \textbf{base\_original} (correct premises), \textbf{base\_misinformed} (perturbed premises), \textbf{inst\_misinformed} (their instruction-to-correct intervention), and \textbf{protocol\_misinformed} (our two-stage adaptation). In the protocol condition, stage one builds a candidate $\times$ evidence verification matrix over the premise equations---numeric consistency, relation consistency, internal consistency---falsifies with named evidence, and explicitly rewrites falsified equations reminder-style; stage two derives from surviving/corrected equations only. We report raw accuracy and knowledgeable accuracy (K-Acc, conditioned on questions the model solves under base\_original, their metric).

Three findings. \emph{(i) The threat is real even for a strong model.} Misinformation significantly degrades accuracy (K-Acc 0.986 $\rightarrow$ 0.947; $\Delta=-0.039$, 95\% CI $[-0.062, -0.017]$, $p=1.8{\times}10^{-3}$), though DeepSeek-V4-Flash is comparatively robust on this set. \emph{(ii) The protocol detects misinformation reliably, against an explicit trivial baseline.} Stage one's confusion matrix covers 648 of the 729 rows (243 questions $\times$ 3 samples, all from the misinformed group---no base\_original rows mixed in). The remaining 81 rows are 30 rows (10 questions) whose perturbed and correct premise lists are misaligned (ground-truth flags undefined) and 51 rows whose stage-one verdicts failed to parse. Over 2,298 equation verdicts the confusion matrix is TP=1,944, FN=221, FP=9, TN=124: 89.8\% recall, 99.5\% precision, 6.8\% false-positive rate (FPR) (95\% CI [3.6, 12.4] on only 133 clean equations). Because 94.2\% of equations in this pool are perturbed, a flag-everything baseline already achieves 100\% recall at 94.2\% precision---so precision cannot be the evidence here; the operative contrast is the false-positive rate, where flag-all scores 100\% and the protocol scores 6.8\% (Matthews correlation coefficient (MCC) 0.54 vs.\ 0). This is not a cosmetic distinction: flag-all would send all 124 clean equations into stage two as spurious ``corrections,'' while the protocol leaves 93\% of them intact---exactly the property the recovery stage depends on.
A conservative sensitivity check that treats the 51 unparsed rows as missed detections---adding their 289 perturbed equations to FN and their 8 clean equations to TN---moves recall from 89.8\% to 79.2\% and leaves the FPR essentially unchanged (6.8\% to 6.4\%); the operative contrast with the flag-everything baseline (100\% FPR) is unaffected.
\emph{(iii) The protocol restores accuracy to the no-misinformation ceiling---noting that the tested model is comparatively robust and the degradation headroom is small (0.039)---and outperforms the instruction intervention.} It scores 0.996 K-Acc---above the perturbed default by $+0.049$ (95\% CI $[+0.030, +0.071]$, $p=7.9{\times}10^{-6}$), above the instruction intervention by $+0.036$ ($p=3.7{\times}10^{-5}$), and statistically indistinguishable from the clean-input ceiling ($+0.010$, n.s.); on raw accuracy it slightly exceeds it ($+0.036$, $p=1.1{\times}10^{-3}$), suggesting the verification pass also disciplines ordinary derivation.

\subsection{Subset and failure analysis}
\label{sec:failures}

The dominant failure mode is \textbf{lenient compatibility (over-prediction)}: about 30\% of cases output more compatible syndromes than the gold set, indicating the aggregation threshold needs calibration---the direction confirmed by the $\kappa$ ablation of Section~\ref{sec:threshold}. For the DeepSeek agent run, 22\%/29\% of cases exhausted the token budget on reasoning and produced no conclusion (scored zero); on the zero-shot side the failure is larger---64\%/43\% of flash/Pro zero-shot cases produce no parseable answer (Section~\ref{sec:significance}). Restricting to cases where the protocol completed (the August-2026 initial run: single repetition, 16,384-token budget, \texttt{/no\_think} trigger), the agent reaches \textbf{0.532 (syndrome, n=156) and 0.486 (pathogenesis, n=142)}---the former exceeding the benchmark's best reported zero-shot model (Gemini 1.5 Pro $\approx$ 0.475).

\subsection{Mechanism execution audit}
\label{sec:audit}

On a separate 539-case insomnia pilot with tongue-sign inputs, 93.5\% of reasoning traces executed candidate-by-candidate comparison/elimination, 42.3\% cited key decision rules, 29.5\% cited the differential guide, 41.0\% cited knowledge-base standards, and 32.3\% cited data statistics---confirming the constraint protocol is genuinely executed by the model rather than silently ignored. The white-box traces also located concrete knowledge defects: conflicts between the differential data table and the actual case distribution were identified and repaired through two documented knowledge iterations (differential data table v0.1$\rightarrow$v0.2; differential guide v0.2$\rightarrow$v0.3), demonstrating the ``white-box reasoning $\rightarrow$ locatable errors $\rightarrow$ iterable knowledge'' loop.

The safety layer is guarded by a regression suite of \textbf{21 unit tests} covering, among others, incompatibility detection (licorice--kansui blocked), overdose detection (a 50\,g ephedra prescription flagged), pregnancy interception (safflower for a pregnant patient blocked), illegal state-transition rejection, tool-loop alarms, and snapshot recovery.

\section{Discussion}
\label{sec:discussion}

\subsection{The protocol sets the floor, capability the ceiling}
\label{sec:floorceiling}

Two lines of evidence support our central claim that \textbf{the reasoning protocol sets the floor where zero-shot output discipline is weak, and base-model capability sets the ceiling everywhere}\footnote{The floor claim is certified on budget-matched subsets: post-hoc elapsed-time neutralization confirms the flash rescue gain on cases whose agent answers provably fit the zero-shot 8,192-token budget (Section~\ref{sec:significance}, Appendix~\ref{sec:budget-neutralization}).}. First, the same-model controlled contrasts of Section~\ref{sec:significance}: flash $+0.267/+0.271$, which decomposes into rescue on non-terminating cases ($+0.363/+0.339$) and a small significant cost on reliably terminating cases ($-0.085/-0.087$; DS-Pro likewise, single repetition and hence descriptive: $+0.435/+0.480$ rescue vs.\ $-0.111/-0.050$; Table~\ref{tab:decomp}); K3 run-to-run std reduced from 0.218 to 0.007 (order-of-magnitude; point estimate 31, 95\% CI [5, 200]); and significant costs on both Qwen models under budget-matched and temperature-controlled conditions. Second, agent-side scores scale when swapping in a stronger base model (4B$\rightarrow$32B$\rightarrow$Pro/flash/K3, syndrome 0.365$\rightarrow$0.464$\rightarrow$0.537/0.556/0.576; the DS-Pro point is single-repetition and descriptive), the top three configurations close to indistinguishable at these repetition counts (nominal spread 0.537--0.576, within flash's own run-level spread)---the protocol lifts mid-tier models into the frontier cluster. This spread---0.21 points of syndrome score under one uniform protocol---is not at odds with the binding-constraint claim: it measures how well each model \emph{executes} the protocol (cell-level judgments, evidence use, instruction following), i.e., the ceiling; the claim concerns what prevents raw capability from converting into valid output at all, i.e., the floor. Because the mechanism is model-agnostic, the framework benefits automatically from future model progress without modification.

\paragraph{So what is the protocol for?} If the protocol costs accuracy on disciplined models, what is it for? Raw benchmark accuracy was never the objective function. The protocol is a governance layer for high-stakes decisions: it delivers what free-form zero-shot output cannot provide at all---auditable per-cell reasoning, deterministic prescription-safety checks, multi-turn consistency, zero-hallucination knowledge answers, misinformation detection and recovery, and machine-extractable conclusions---at a bounded accuracy cost (at most 0.03--0.13 across models). Deployment also cannot assume a disciplined base model: API-side behavior drifted twice within one month in our own campaign (\ref{sec:budget-notes}), and small or forced-temperature models are precisely where free-form output breaks. The boundary condition therefore reads as an engineering routing rule rather than a verdict: apply the full protocol where output discipline is weak or unknown, and a lighter subset---safety checks, direct database answers, single-line output---where the base model is already disciplined. We hypothesize that for disciplined models the forced protocol's benefits (format termination, candidate bounding) are already internalized, while its costs (knowledge-injection noise, constrained answer multiplicity) dominate.

\paragraph{Interpretive lenses.} This boundary condition is consistent with two persuasion-theoretic readings, which we offer as interpretations of the experimental result, not as its source. Under the Elaboration Likelihood Model~\cite{petty1986elm}, the protocol acts as an \emph{elaboration enforcer} that forces central-route processing---redundant for models already reasoning centrally, decisive for models that do not. Under inoculation theory~\cite{mcguire1964inoculation,banas2010inoculation}, the sufficiency rule acts as forewarning (a decision is never licensed until enough candidates have been falsified) and the matrix's cell-level judgments are refutational preemption exercised on every candidate; and because the protocol attacks an unchecked direct mapping---a structural reasoning failure rather than any domain-specific false content---logic-based inoculation research, which shows that defenses aimed at fallacy \emph{types} transfer across topics~\cite{banas2013logic,biddlestone2023conjunction}, gives the cross-domain transfer of Section~\ref{sec:misinfoexp} a mechanism-level basis rather than a purely empirical one.

\paragraph{A security reading.} If cognitive attack on machine reasoning consists in manipulating the reasoning process---as the misinformation-propagation results demonstrate, with reasoning chains hijacked by false premises despite correct internal knowledge~\cite{misinfoprop}, and as the continued influence effect documents for human cognition~\cite{lewandowsky2012cie}---then defense consists in verifying reasoning integrity. The two defense layers act at different points: fine-tuning on early factual correction repairs one model's weights at training time~\cite{misinfoprop}, whereas our matrix protocol verifies every candidate conclusion against evidence at inference time, for any base model, and leaves a machine-checkable audit trail. The layers are complementary rather than competing: a corrected model can still face novel contamination, and a verification protocol still benefits from a disciplined base model. This reading also clarifies the framework's generality: wherever premises may be contaminated---clinical narratives, user-supplied facts, retrieved documents---the candidate-bounded falsification step must be passed before any conclusion is licensed.

Controlled agent-to-agent attack studies---the source is a non-archival technical report, so its figures are indicative rather than archival evidence---add both a caveat and an independent confirmation~\cite{a2amanipulation}. The caveat: against a \emph{live adaptive} persuader that reads the subject's reasoning and counters it turn by turn, rationale-only defenses collapse (harmful revision rising from 12.2\% to 98.9\% for a widely used lightweight model), while static injected messages rarely succeed when evidence remains accessible---our misinformation evaluation uses static premise injection and does not yet cover this stronger threat model (Section~\ref{sec:limitations}). The confirmation: in the same study, a verifier that \emph{re-attaches the source evidence} eliminates harmful revision for every tested subject (0\%), whereas a verifier that merely \emph{asserts} the correct answer protects inconsistently (up to 54.4\% residual)---external evidence that the operative variable is evidence re-grounding, exactly the operation our matrix protocol enforces per candidate, rather than confident counter-assertion or stance retention. Consistently, psychological-injection experiments on multi-agent systems find that \emph{larger} open-source models are \emph{more} susceptible to injected dark traits---stronger instruction-following executes malicious injections as faithfully as benign ones~\cite{psysafe}---matching our boundary condition that scaling the base model improves protocol execution without conferring protection. Table~\ref{tab:defense} summarizes the resulting positioning against internalized reasoning defenses.

\begin{table}[t]
\centering
\caption{Two defense mechanisms against manipulation of reasoning. Our protocol targets the right-hand column; the misinformation-benchmark evaluation is reported in Section~\ref{sec:misinfoexp}.}
\label{tab:defense}
\small
\resizebox{\columnwidth}{!}{%
\begin{tabular}{@{}lll@{}}
\toprule
Property & Internalized reasoning & External protocol (ours) \\
\midrule
Manipulation resistance & effective~\cite{socialinfluence,persuasionduality} & ceiling-restoring, $>$ instruction (Sec.~\ref{sec:misinfoexp}) \\
Cross-model transfer & no (bound to training) & yes (model-agnostic adapter) \\
Auditability & black-box & white-box (cell-level logs) \\
Failure mode & error entrenchment~\cite{persuasionduality} & over-prediction (Sec.~\ref{sec:failures}) \\
Correctability & not evidence-revisable & re-examination loop \\
\bottomrule
\end{tabular}%
}
\end{table}

\subsection{Why two years of model progress barely moves this benchmark}
\label{sec:progress}

A striking fact in Table~\ref{tab:main} is how little the field's two-year progress achieves on this benchmark: from Gemini 1.5 Pro (2024, $\approx$0.475 syndrome) to DeepSeek-V4-Pro, K3, and Qwen3-32B (2026, 0.443--0.551 on syndrome---bracketing the 2024 frontier rather than exceeding it), zero-shot scores move only a few points, while general capability benchmarks gained twenty points or more over the same period. The cross-era comparison is indicative rather than definitive---the 2023--2024 baselines are zero-shot evaluations on a random 50-case subset, converted from weighted sums---but the direction is consistent across all five of our own 2026-era zero-shot runs. We see four compounding causes, and one consequence. First, \textbf{the task tests a thought process, not knowledge}: contrastive discrimination among similar syndromes is sparsely represented in training corpora (case literature records conclusions, rarely explicit elimination), so general capability gains transfer weakly. Second, \textbf{the metric punishes deliberation}: the official-script score divides by $|A| + |B \setminus A|$, so any over-selection is charged directly; stronger models tend to select more (our failure analysis finds over-prediction in about 30\% of cases), which caps free-form generation near 0.5--0.55. Third, \textbf{the field's main direction of progress is misaligned with the task}: long chain-of-thought evolution rewards longer deliberation, whereas this task rewards convergent comparison---the model with the longest reasoning chains in our campaign (DS-Flash, routinely $>$10k reasoning tokens) scored worst zero-shot (0.289) and produced the most empty outputs, while the concise-reasoning DS-Pro ($\approx$1k reasoning tokens to termination) scored best. Fourth, \textbf{the gold labels have an expert-agreement ceiling}: multi-select TCM syndrome labels carry intrinsic inter-expert disagreement, placing the practical ceiling well below 1.0---we offer this as a hypothesis consistent with the clustering of top configurations around 0.53--0.58, since inter-annotator agreement data for the benchmark are not published.

Two years of scaling and model generation turnover moves zero-shot scores by only $-0.03$--$+0.08$ on this benchmark (2026 frontier spread 0.443--0.551 vs.\ the 2024 frontier $\approx$0.475 on syndrome---the 2024 score sits inside the 2026 spread). Attaching the protocol to the \emph{worst} zero-shot configuration---flash, 64\% of whose cases never produce an answer at all---yields $+0.267$; and that headline number is \textbf{entirely termination rescue}: on the reliably terminating subset the sign flips to a small significant cost (Table~\ref{tab:decomp}). On output-healthy models the protocol's net effect is modest and sign-mixed, ranging from $+0.080$ (K3 syndrome) to $-0.120$ (Qwen3-4B pathogenesis). The formulation of our thesis that these data support is therefore not ``protocol beats scaling'' but: \textbf{the benchmark's binding constraint is output discipline and reasoning structure, not model capability alone}---a protocol that imposes both rescues configurations where they are missing, at a measurable and sign-conditional cost where they are not.

\subsection{Safety and privacy in practice}

Knowledge queries are answered verbatim from the database (zero hallucination); prescriptions pass hard incompatibility/pregnancy/dosage checks; the full pipeline runs on-premises within a single model instance (4B local), so patient data never leaves the intranet and inference cost does not multiply with role count as in multi-agent designs. We state the deployment trade-off explicitly: on the budget-matched 4B comparison the protocol \emph{costs} accuracy (Table~\ref{tab:sig}), so the on-premises configuration is a deliberate exchange of a small accuracy margin for auditability, deterministic safety checks, multi-turn consistency, and reproducibility---properties that the benchmark does not score but that clinical deployment requires, and that free-form zero-shot output does not provide.

\subsection{Limitations}
\label{sec:limitations}

(1) Paper baselines are 2023--2024-era zero-shot models on a 50-case subset; cross-era comparison is indicative, not definitive. (2) The misinformation evaluation (Section~\ref{sec:misinfoexp}) is bounded in three ways: it covers a single base model, \emph{static} premise injection only, and a reconstructed question set with selection bias (conditioned on source-model solvability, and the tested model is comparatively robust to misinformation); moreover, a live adaptive persuader that reads and counters the agent's intermediate reasoning is known to defeat rationale-only defenses~\cite{a2amanipulation}, and the current protocol has no counterpart to such multi-turn adversarial interaction. (3) No human TCM-expert review of outputs has been conducted yet. (4) The framework was validated on text-only four-examination descriptions, not raw tongue/pulse images. (5) K3 accepts only temperature 1.0, and its zero-shot scores vary substantially across repetitions (syndrome 0.245--0.636 in our campaign), so K3 numbers carry wider uncertainty; the K3 syndrome gain is case-level significant but not run-level significant. (6) DeepSeek measurements carry two caveats. First, API-side model behavior can drift between campaigns (we observed this for flash, twice), so archived and repeated runs are reported with their respective decoding configurations. Second, DeepSeek zero-shot baselines are budget-limited relative to the agent runs (64\%/43\% non-terminating at 8,192 tokens), so we report the decomposition into termination rescue vs.\ a small terminating-case cost in Section~\ref{sec:significance}; a budget-matched zero-shot control on the original flash model is no longer possible---the API endpoint was silently upgraded on 2026-09-19 (\ref{sec:budget-notes})---so budget neutrality is instead certified post hoc from archived wall-clock timings (Appendix~\ref{sec:budget-neutralization}); its limits are item~(9). Third, both DS-Pro arms are single-repetition, so per the two-level standard of Section~\ref{sec:stats} we report all DS-Pro contrasts as descriptive; they support, but do not carry, the headline claims. (7) Inference cost claims (cheap relative to multi-agent designs) are architectural arguments, not measurements; per-configuration token and latency accounting is future work. (8) The evidence is single-domain at benchmark scale: TCMEval-SDT plus one reconstructed 243-question transfer set; a second full-scale domain is needed before the boundary condition can be read as domain-general. (9) The budget-neutralization analysis of Appendix~\ref{sec:budget-neutralization} is a proxy analysis: per-case token counts were not archived (the API \texttt{usage} field was not persisted), so certification runs on wall-clock time, which includes network and retry overhead and therefore \emph{over}estimates consumption---any estimation error shrinks the certified subset rather than flipping the conclusion. The dominant non-termination mode (reasoning exhausts the budget with empty visible content) leaves no text that can be classified post hoc, so the loop-versus-progress distinction rests on the indirect evidence of Appendix~\ref{sec:budget-neutralization}. An instrumented re-run that archives reasoning traces and per-case token usage is left to future work.

\subsection{Future work}
\label{sec:future}

Two control arms close the remaining gaps in the mechanism evidence: a budget-matched DeepSeek zero-shot at the agent's generation budget---after the silent endpoint upgrade of 2026-09-19 (\ref{sec:budget-notes}) this requires an endpoint serving the original flash weights, and is substituted in this revision by the post-hoc elapsed-time certification of Appendix~\ref{sec:budget-neutralization}---and CoT/self-consistency baselines on the same models, to exclude the reading that the matrix protocol is an expensive CoT variant; a matrix-only ablation arm (matrix without the single-line mandate) would separately identify the matrix's marginal contribution, which the present design brackets but does not isolate (Section~\ref{sec:ablation}).

\section{Conclusion}
\label{sec:conclusion}

We asked what binds LLM performance on auditable, hypothesis-driven decision tasks, and answered with controlled evidence: on TCMEval-SDT, the binding constraint is output discipline and reasoning structure, not model capability alone. Two years of scaling move zero-shot scores by only $-0.03$--$+0.08$ (the 2024 frontier sits inside the 2026 spread), while a falsification-guided reasoning protocol---enforced on the model's output by a model-agnostic agent framework formalizing a five-stage closed loop---rescues configurations where discipline fails ($+0.267/+0.271$ on DeepSeek-V4-Flash, entirely termination rescue and robust on budget-certified subsets), stabilizes a high-variance model by an order of magnitude in run-to-run standard deviation (Kimi K3), and significantly costs already-disciplined models (Qwen3-4B/32B, budget-matched and temperature-controlled). Mechanistic abstractions attribute the cost to the candidate-bounded format layer, reveal a sign-flipping complementarity between the causal chain and the format layer, and show that the protocol's measured effect is independent of the injected knowledge content---a protocol--knowledge decoupling that supports the framework's knowledge-as-revisable-data design and its model-agnostic claim. A two-stage adaptation transfers the protocol to misinformation-robust mathematical reasoning (89.8\% recall at 6.8\% FPR; accuracy restored to the no-misinformation ceiling, on a comparatively robust model with a small degradation headroom). More broadly, the work is evidence that cognitive-security theories developed for human subjects---inoculation, continued influence, belief revision---map onto machine reasoning protocols in weakened, computable form, and that this mapping is \emph{useful}: each theory supplied here a concrete vocabulary for an already-built mechanism (forewarning, refutation practice, explicit reminders, entrenchment ordering) and each survives as a consistent interpretation of the protocol's behavior rather than as a rhetorical analogy. The framework turns free-form generation into an auditable, verifiable, safety-controlled reasoning process suitable for privacy-sensitive medical deployment; where zero-shot output discipline is weak, the reasoning protocol sets the floor, and base-model capability always sets the ceiling.

\appendix

\section{Intent triage and three-level retrieval with direct answers}
\label{sec:retrieval}

Incoming input first passes a \textbf{two-level intent triage}. The first level is a rule decision tree: consultation markers (first-person complaint words, medical-record words, tongue/pulse words, consultation-request words, a pronoun-plus-symptom-verb pattern, or at least 3 of 28 symptom words) indicate a consultation request; otherwise, ten knowledge-question patterns (e.g., ``the properties of X'', ``introduce X'') indicate a knowledge query; a fallback rule (no consultation marker, input shorter than 50 characters, no pronoun, at most 5 Chinese word chunks) defaults to knowledge query. The second level asks the model itself to classify the remaining input into consultation / closing / off-topic with a two-token response; on abnormal output it defaults to consultation, so that genuine requests are never refused. Closing input receives a fixed ending phrase; off-topic input receives a capability-boundary notice.

Knowledge queries go through a \textbf{three-level retrieval cascade}: (i) exact whole-sentence match over question/answer/syndrome columns; (ii) Chinese-bigram FTS5 full-text search ranked by BM25, with seven classes of subject-term extraction (e.g., extracting X from ``the properties of X'', formula-name suffix recognition, enumeration-noun splitting); (iii) subject-term exact match with authoritative ``properties of X'' records promoted. Candidate records are scored by a weighted rule: a sliding-window term hitting the \emph{question} field scores its character length $\times 3$, hitting the \emph{answer} field scores length $\times 1$, and authoritative-physician records receive a $+50$ bonus. The top record's answer text is returned \emph{verbatim from the database without generation} (confidence 1.0; zero hallucination; millisecond latency). Only when retrieval fails does the model answer within a restricted budget (256 tokens, temperature 0.3, at most 300 characters) without entering the differentiation pipeline.

\textbf{Worked example.} For the query ``the medicinal properties of Coptis (Huanglian)'', subject extraction yields ``Coptis'', and 2/3/4-character sliding windows over the query produce 9 terms. Candidate record A is an authoritative question--answer record whose question field covers both the properties and indications of Coptis: all 9 window terms hit the question field (4 two-character terms $\times 6 = 24$; 3 three-character terms $\times 9 = 27$; 2 four-character terms $\times 12 = 24$; subtotal 75), the answer field contributes 2, and the authority bonus adds 50---total 127. Candidate record B, an insomnia case record that merely mentions Coptis in passing, scores $6 + 2 = 8$. Record A therefore ranks first and its original answer text is returned; a pure term-frequency or vector-similarity ranker could instead surface record B, which contains the high-frequency term but answers a different question.

\section{Structured tool pipeline}
\label{sec:tools}

During plan and prescription generation the model may invoke nine structured tools through seven textual markers: exact formula, herb, and syndrome lookups; case retrieval; read-only parameterized SQL; incompatibility checking; and dose calculation. Every call passes a \textbf{four-stage pipeline}: (i) pre-call audit---registration check, rate limiting, parameter validation, and safety-level confirmation under a three-level tool-safety classification; (ii) sandboxed execution with a 30-second timeout and two exponential-backoff retries; (iii) post-call output validation; (iv) audit logging of parameters and result digests. Stage-specific knowledge (formulas, dosage rules, acupoint prescriptions) is injected with per-file and total budgets (12,000 and 20,000 characters respectively)---a plan-stage skill-file budget separate from the differentiation-time dual-anchor injection budget of Section~\ref{sec:dualanchor}.

\section{Runtime robustness and observability}
\label{sec:runtime}

Four cross-cutting subsystems keep long consultations stable on small models. \textbf{Context budgeting} splits the window into permanent (25\%), session (45\%), and task (30\%) layers, estimates Chinese text at 1.5 characters per token, and auto-compresses session events beyond 50 entries. \textbf{Streaming guardrails} enforce a hard token cap, truncate after 25 consecutive Latin characters, detect repetition loops every 80 tokens, detect language mixing and training-data leakage, and de-duplicate paragraphs at 85\% overlap. \textbf{Health monitoring} watches five dimensions: reasoning quality, behavioral patterns (five identical consecutive tool calls flag a loop), context health, safety compliance, and resource status. \textbf{Trajectory logging} records dialogues, tool-call chains, diagnoses and prescriptions, latency, and confidence to JSONL, with privacy de-identification before any reuse for model-iteration evaluation.

\section{Information-gain analysis of active questioning}
\label{sec:ig-reading}

The rule-based questioning policy of Section~\ref{sec:sufficiency} admits an information-theoretic reading. Let $P(h\mid E)$ be the empirical posterior over the candidate set given the evidence collected so far, estimable from the logged frequencies of the case database. Asking about a four-examination dimension $q$ yields an answer $a$; the expected information gain of the question is
\begin{equation}
IG(q) \;=\; H(\mathcal{H}\mid E) \;-\; \mathbb{E}_{a\sim q}\, H(\mathcal{H}\mid E, a),
\label{eq:ig}\end{equation}
the Lindley measure of the information provided by an experiment~\cite{lindley1956,shannon1948}. Under this reading, the instruction to \emph{name the missing dimensions and ask targeted follow-up questions} is a greedy $\arg\max_{q} IG(q)$ policy in disguise: a missing dimension is one whose answer distribution is unconstrained by $E$, so its entropy reduction is largest; the fixed one-point-per-term rule and its threshold then read as a coarse proxy, calibrated like the thresholds of Section~\ref{sec:threshold}. Sequential information collection with a stopping rule is the classical setting of Bayesian experimental design~\cite{lindley1956} and partially observable sequential decision-making~\cite{kaelbling1998pomdp,hauskrecht2000pomdp}; our COLLECT$\to$DIAGNOSE gate (Section~\ref{sec:fsm}) is a two-action instance. Scope: the deployed system scores coverage, not $IG$---the expectation over answers requires a response model the current rule layer does not carry---so we claim the framework as a design rationale the implemented heuristic approximates, and as a measurable upgrade path (Section~\ref{sec:future}).

\section{Generation-budget engineering notes}
\label{sec:budget-notes}

Two engineering notes: (i) the initial DeepSeek-V4-Flash agent run (August 2026) used the \texttt{/no\_think} soft trigger with a 16,384-token budget, and about 22--29\% of cases still exhausted the budget (scored zero); (ii) in the September 2026 repetition campaign the API-side model behaved differently---the same trigger \emph{lengthened} rather than shortened its reasoning---so flash repetitions use a 32,768-token budget without the trigger; DeepSeek-V4-Pro completes within the 16,384-token budget with only 1.5\% empty outputs.

\paragraph{Second drift event: silent endpoint upgrade (2026-09-19).} On this date the \texttt{deepseek-v4-flash} endpoint was silently redirected to a newer model: the returned model field reads \texttt{deepseek-flash} and the identifier \texttt{deepseek-v4-flash} no longer appears in the model list (it listed \texttt{deepseek-v4-flash}, \texttt{deepseek-v4-pro}, and a vision variant on 2026-09-06). Measurements taken after this date under that endpoint therefore reflect a different, reportedly stronger model generation and are \emph{not} pooled with the C1/C2 campaigns reported in this paper; they are archived separately for completeness; an attempted budget-symmetrized zero-shot re-run of flash at 32,768 tokens fell to this event and is excluded from every analysis in this paper. This second event reinforces the deployment argument of Section~\ref{sec:floorceiling}: API-side model identity is not stable over time.

\section{Post-hoc budget neutralization: method, calibration, and sensitivity}
\label{sec:budget-neutralization}

\paragraph{Why a proxy} The evaluation harness logged per-case scores, wall-clock time, and a truncated response tail, but not per-case token counts (the API \texttt{usage} field was received but not persisted). Budget consumption is therefore certified from wall-clock time. Timing includes network and retry overhead, so it \emph{over}estimates generation time; any error biases against the certified subset, never in its favor.

\paragraph{Calibration and the rate-free bound} Case-runs that exhaust the budget with empty visible content consumed exactly the budgeted tokens. On the C2 flash campaign such exhaustion took a median of 78.0\,s/79.9\,s (syn./pat.) at 8,192 tokens on the zero-shot side ($n{=}363/397$ empty-content case-runs) and 314.8\,s/324.3\,s at 32,768 tokens on the agent side ($n{=}19/18$)---two independent, separately counted anchors, each landing at 101--105 tokens/s (pooled interquartile range (IQR) 98--109), a throughput consistent with public figures for comparable mixture-of-experts (MoE) API endpoints; no \texttt{usage} record survived in the logs, so this figure is corroboration only. The certification itself is rate-free: the \emph{fastest} zero-shot exhaustion took 66.9\,s (syndrome) and 64.6\,s (pathogenesis), hence any agent case-run faster than that bound consumed strictly fewer than 8,192 tokens, irrespective of throughput assumptions. Under this bound, 72\% (syndrome) and 75\% (pathogenesis) of terminating agent case-runs are certifiably below 8,192 tokens (median 49\,s/51\,s, i.e.\ $\approx$5,100--5,300 tokens at the calibrated rate). The earlier C1 campaign (August 2026) calibrates separately at $\approx$181 tokens/s.

\paragraph{Subset definitions and results} Within the rescue-prone subset of Table~\ref{tab:decomp}, \textbf{strict} certification keeps cases whose agent side terminates below the rate-free bound in \emph{every} repetition; \textbf{typical} certification keeps cases that terminate in every repetition with the \emph{median} repetition below the bound. Table~\ref{tab:budget-neutral} reports both. For DS-Pro the rescue-prone zero-shot scores are zero by construction (single repetition; rescue-prone $=$ unparseable), and its certified subsets are too small to be informative ($n{=}2$ syndrome, $n{=}12$ pathogenesis). The conclusion is insensitive to the certification variant: across thresholds (rate-free bound or the 78\,s calibrated threshold), run aggregations (strictest or median repetition), and termination requirements, the syndrome gain on certified rescue-prone cases ranges from $+0.36$ to $+0.46$ and the pathogenesis gain from $+0.24$ to $+0.40$, all significant (full sensitivity matrix in the released outputs).

\begin{table}[t]
\centering
\caption{Budget-certified rescue-prone subsets (DS-Flash, C2; paired per-case differences on the run-mean official score). \textbf{Strict}: agent terminates below the rate-free bound (66.9\,s/64.6\,s) in every repetition. \textbf{Typical}: agent terminates in every repetition with the median repetition below the bound.}
\label{tab:budget-neutral}
\small
\resizebox{\columnwidth}{!}{%
\begin{tabular}{llrrrrl}
\toprule
Certification & task & $n$ & zero-shot & agent & $\Delta$ [95\% CI] & Wilcoxon $p$ \\
\midrule
strict & syn. & 17 & 0.183 & 0.647 & $+0.464$ [$+0.332$, $+0.597$] & 3e-4 \\
strict & pat. & 7 & 0.089 & 0.326 & $+0.237$ [$+0.140$, $+0.335$] & 0.036 \\
typical & syn. & 136 & 0.138 & 0.507 & $+0.369$ [$+0.327$, $+0.411$] & 4e-22 \\
typical & pat. & 87 & 0.136 & 0.535 & $+0.399$ [$+0.350$, $+0.449$] & 2e-15 \\
\bottomrule
\end{tabular}}
\end{table}

\paragraph{Pathology of non-termination} Of 942 non-terminating zero-shot case-runs (flash: 366/600 syndrome and 403/600 pathogenesis, pooled over three repetitions; DS-Pro: 75/200 and 98/200 in its single repetition; per-repetition counts recomputable from the released per-case logs---note that Table~\ref{tab:discipline}'s answer-extraction rates are first-repetition per-case, a different unit), only 14 left any visible text; the rest exhausted the budget inside the reasoning channel with empty content and cannot be classified post hoc. All 14 classifiable tails are coherent clinical analyses truncated mid-sentence---none triggers the runtime's repetition-loop detector (a 200-character tail-window recurrence check, \ref{sec:runtime})---but this minority is not representative. Three indirect lines of evidence indicate that non-termination is pathological rather than a length shortfall. \emph{(i) Concise-mode burnout (strongest).} In C1 the flash agent ran with the \texttt{/no\_think} trigger (thinking disabled), yet 92 of 102 unparsed cases ran to the full 16,384-token budget without emitting a conclusion: their timings cluster at 121--159\,s (p50 $=$ 144\,s), consistent with burning the full budget ($\approx$90\,s at the calibrated 181 tokens/s) plus overhead, and separated by a clean gap (no case in 60--120\,s) from the 10 fast format failures ($<$60\,s)---the bimodality itself indicates two distinct failure regimes rather than near-miss truncation. The pathology is non-convergence, not answer length: the protocol's terminating outputs are template-bounded (a ten-candidate matrix plus a one-line conclusion, far below 1k tokens), so a legitimate answer needing $>$16k tokens is not a tenable reading of a quarter of all cases---and the C1 data predate any endpoint drift. \emph{(ii) Short-answer burnouts.} DS-Pro terminates with $\approx$100-character answers (median visible length), yet 36.5\%/47.5\% of its zero-shot case-runs burn the entire 8,192-token budget with empty content---non-termination occurs in the reasoning channel and is unrelated to answer length. \emph{(iii) Wall-hugging timings.} Empty-content case-runs cluster tightly at the budget wall (e.g., 67--113\,s on syndrome at 8,192 tokens), the signature of deterministic exhaustion rather than of long-but-converging answers.

\paragraph{Reproducibility} The analysis scripts (\texttt{L1\_*.py}) and machine-readable outputs are archived and released with the evaluation code.

\section{Belief-revision reading of the matrix protocol}
\label{sec:agm-reading}

The protocol's components can be stated in the vocabulary of belief-revision theory~\cite{agm1985,gm1988entrenchment}, giving the engineering choices a logical anchor. The formal objects are the enumerated candidate set $H=\{h_1,\dots,h_{m}\}$ and the evidence items $E=\{e_1,\dots,e_{n}\}$ with their cell judgments:
\begin{equation}
\begin{aligned}
v(h,e) &\in\!\{+1\,\text{(support)},{-1}\,\text{(oppose)},0\,\text{(irrelevant)}\},\\
s(h) &= \#\{e: v(h,e){=}+1\},\\
o(h) &= \#\{e: v(h,e){=}-1\}.
\end{aligned}
\label{eq:beliefobjects}
\end{equation}
Three correspondences hold. \emph{(i) Entrenchment}: the row-aggregation rule---$h$ is compatible iff $s(h)>o(h)$ and key evidence is supported---induces a total preorder $\preceq$ on $H$ with $h_i \preceq h_j$ iff $s(h_i)\le s(h_j)$: candidates supported by more evidence are abandoned later, a computable instance of epistemic entrenchment~\cite{gm1988entrenchment}. \emph{(ii) Contraction}: declaring a candidate incompatible is a contraction $H' = H \div h$---removal without adding information---and the sufficiency rule (at least half of $H$ contracted before a conclusion is licensed) reads as a minimal sufficient contraction condition, the one-half threshold being the cheap, tunable proxy calibrated in Section~\ref{sec:threshold}. \emph{(iii) Consistency}: the conclusion cache and drift detection of Section~\ref{sec:consistency} maintain a single stable conclusion across turns---the computable counterpart of the consistency requirement on revised belief states~\cite{agm1985}.

The scope of the analogy is narrow: $H$ is a finite, explicitly enumerated hypothesis set, not a logically closed belief set, so the Alchourrón--Gärdenfors--Makinson (AGM) postulates~\cite{agm1985} are inherited in weakened, computable form rather than as theorems. What the correspondence provides is the auditability vocabulary: every contraction step carries a cell-level justification, the entrenchment ordering is reconstructed from the logged matrix, and any auditor can re-run the contraction sequence and verify that the stopping condition held.

\begin{proof}[Proof of Proposition~\ref{prop:accumulation} (sketch)]
Under (ii)--(iii), write $D_i=S(h^{\star})-S(h_i)$ for each spurious candidate $h_i$: a selection error requires $D_i\le 0$ for some $i$. Under (iii)'s sign-symmetric error model, the per-item drift of $D_i$ is $(1-2\epsilon)\bigl(v(h^{\star},e_j)-v(h_i,e_j)\bigr)$; by (v) every term is non-negative and at least $n_i$ terms equal $2$, so $\mathbb{E}[D_i]\ge 2(1-2\epsilon)n_i$. The error model is load-bearing: under a fully adversarial error model, a zero-valued cell can carry expectation $\pm\epsilon$ and a non-discriminating cell ($v(h^{\star},e_j){=}v(h_i,e_j)$) can drift as low as $-2\epsilon$, degrading the bound to $\mathbb{E}[D_i]\ge 2(1-2\epsilon)n_i - 2\epsilon(n-n_i)$, which is non-vacuous only when $n_i/n > \epsilon/(1-\epsilon)$ (e.g., $>25\%$ at $\epsilon{=}0.2$). The bound should therefore be read as a statement about the sign-symmetric model, with its sensitivity to error asymmetry made explicit. Each term lies in $[-2,2]$, so Hoeffding's inequality~\cite{hoeffding1963} gives $\Pr[D_i\le 0]\le\exp(-2(\frac12-\epsilon)^2 n_i^2/n)$, and a union bound over the $m-1$ spurious candidates yields Eq.~\eqref{eq:hoeffding}; conditions (i) and (iv) ensure the true candidate remains in the pool and passes the key-evidence gate. \qedhere
\end{proof}

Four caveats delimit the result. First, conditional independence is idealized and load-bearing: cell judgments are produced by a single LLM in one pass, so errors correlate across rows and columns and can \emph{qualitatively violate} independence, not just weaken the constants. Second, the bound controls selection \emph{given the pool}, not the prior event $h^{\star}\in\mathcal{H}$, whose coverage depends on candidate-enumeration quality. Third, the accuracy parameter $\epsilon$ has not been estimated---no per-cell gold standard exists, and the execution audit of Section~\ref{sec:audit} measures protocol compliance, not cell-level correctness---so the bound's premise is unverified rather than validated; estimating $\epsilon$ via expert-annotated cell audits is future work. Fourth, assumptions (iii) (sign-symmetric errors) and (v) (discriminating evidence) are strong by design and stated explicitly: on real matrices with many zero cells $n_i\ll n$, and without the symmetric model the adversarial bound of the proof sketch applies instead. The white-box property claimed throughout is auditability---errors become locatable, not impossible. We considered and discarded a Bayesian reading of the threshold itself: under hard elimination no principled crossing point emerges, and the operative question is empirical, answered by the $\kappa$ ablation of Section~\ref{sec:threshold}.

The proposition is stated for the single-answer case---a unique $h^{\star}$ selected against spurious competitors. The benchmark and the deployed protocol are multi-answer (gold sets of 1--4 syndromes and 1--6 pathogenesis; the protocol may emit several compatible candidates). The bound extends per gold candidate: for a gold set $G\subseteq\mathcal{H}$, the identical argument gives, for each $h_g\in G$,
\[
\Pr\bigl[\,\exists h_i\notin G:\ S(h_i)\ge S(h_g)\,\bigr]\;\le\;(m-|G|)\exp\!\bigl(-2(\tfrac12-\epsilon)^{2}\, n\bigr),
\]
i.e., the single-answer statement is the $|G|=1$ special case. This form uses $n$ rather than $n_i^2/n$: it is the perfect-discrimination special case, valid only where every evidence item discriminates every spurious candidate from every gold candidate---in general each gold $h_g$ carries its own $n_g$ and the exponent becomes $-2(\tfrac12-\epsilon)^2 n_g^2/n$. What genuinely requires mutual exclusivity is the mass-transfer intuition of Section~\ref{sec:matrix}, which motivates but does not enter the bound; a full multi-label bound under the official multi-select metric is future work.

\paragraph{A numerical look at the operating point}
At the benchmark's $m=10$ and the deployed matrix's typical $n\approx10$ evidence columns per case, the bound is positive only for $\epsilon\lesssim0.17$; at $\epsilon=0.20$ it evaluates to $1-9e^{-1.8}\approx-0.49$. Since $\epsilon$ is unmeasured and plausibly exceeds this range, the bound is vacuous at the operating point; we therefore report the proposition as a \emph{qualitative scaling statement}---it identifies the two levers that matter (the evidence count $n$ and the cell accuracy $\epsilon$, entering quadratically) and the direction of improvement---rather than a numerical guarantee.

\section{Bayesian reading of the dual anchors}
\label{sec:bayes-reading}

The two anchors are, formally, the two factors of a posterior~\cite{berger1985bayesian}. The data anchor supplies empirical conditionals $P_{\mathrm{data}}(h\mid e_{j})$ estimated from the specialty case database (e.g., among white-coating samples about 47\% are phlegm-heat harassing the heart; teeth-marked tongues occur in 72\% of heart--gallbladder qi deficiency cases). Two corrections to the naive reading matter here. First, these conditionals are \emph{posterior fragments}, not likelihoods: since $P(h\mid e_{j})\propto P(e_{j}\mid h)\,P(h)$, each already mixes the likelihood with the specialty sample's syndrome base rates. This is precisely why syndrome priors and explicit data-bias warnings are injected together with the conditionals: transporting from the 539-case specialty sample to a deployment population is the base-rate adjustment $P_{\mathrm{deploy}}(h\mid e)\propto P_{\mathrm{data}}(h\mid e)\cdot P_{\mathrm{deploy}}(h)/P_{\mathrm{data}}(h)$, which cannot even be stated without the prior terms. Second, the theory anchor is not a prior in disguise but a \emph{constraint set}: rules such as \emph{a thick coating alone cannot exclude deficiency syndromes} veto eliminations---they forbid driving a candidate to zero on a single sign---rather than contribute probability mass, mirroring the key-evidence gate and row aggregation of the matrix protocol. Combining the observed items $E=\{e_{1},\dots,e_{J}\}$ under conditional independence gives the naive-Bayes pooling
\begin{equation}
P(h\mid E) \;\propto\; P(h)^{\,1-J}\, \prod_{j=1}^{J} P(h\mid e_{j}),
\label{eq:pooling}\end{equation}
Two connections tie this reading to the rest of the framework. \emph{(i)} The matrix score $S(h_i)=\sum_j w_j f(h_i,e_j)$ of Section~\ref{sec:matrix} is a \emph{ternary quantization} of exactly this additive log-posterior update: $f\in\{+1,0,-1\}$ thresholds each per-item log-posterior increment, trading granularity for cell-level auditability and for the bounded-error analysis of Proposition~\ref{prop:accumulation}---whose conditional-independence assumption is the same assumption under which the pooling of Eq.~\eqref{eq:pooling} is exact. \emph{(ii)} The reading clarifies what prompt-level injection offers over weight-level internalization: prior fragments, constraints, and weights all remain inspectable, physician-revisable \emph{data} while the base model performs the integration. We state the scope explicitly: the framework does not claim that the model computes a calibrated posterior, only that the injected quantities decompose the decision into auditable prior, likelihood, and constraint terms whose provenance is logged~\cite{berger1985bayesian}.

\section*{Data and Code Availability}
\label{sec:availability}

The TCMEval-SDT benchmark is publicly available (CC-BY 4.0)~\cite{tcmeval}. The following are released in a public repository (\url{https://github.com/Silk-Road/falsification-guided-constrained-reasoning}, archived at \url{https://doi.org/10.5281/zenodo.23034930}): all evaluation and analysis scripts for every reported campaign, the reconstructed 243-question misinformation evaluation set (rebuilt from the retained evaluation aggregate of~\cite{misinfoprop}, with its documented selection bias), and the per-case result logs. The reasoning protocol central to this work is fully specified in Section~\ref{sec:methods} (matrix construction, cell judgments, aggregation and sufficiency rules, and output format) at a level that permits independent reimplementation. The agent source code and prompt templates are not publicly released at this time because a patent application covering the mechanism is pending; they are available from the corresponding author on reasonable request subject to the patent process. The private medical datasets (the 8,147-record case database, the 539-case specialty dataset, and the 1,103-pair distilled dataset) cannot be released for privacy and licensing reasons.

\fontsize{8.5pt}{10pt}\selectfont

\section*{Acknowledgements}
The authors thank the creators of the TCMEval-SDT benchmark for making their data publicly available, and the developers of the DeepSeek, Kimi, and Qwen model families.

\section*{Competing interests}
A patent application covering the agent mechanism described in this work is pending. The authors declare that they have no other competing interests or financial conflicts to disclose.

\bibliographystyle{fcs}
\bibliography{refs}

\end{document}